\documentclass{article}

 \usepackage[main, final]{neurips_2026}

\usepackage[utf8]{inputenc} 
\usepackage[T1]{fontenc}    
\usepackage{hyperref}       
\usepackage{url}            
\usepackage{booktabs}       
\usepackage{amsfonts}       
\usepackage{nicefrac}       
\usepackage{microtype}      
\usepackage{xcolor}         
\usepackage{amsmath}         %
\usepackage{amssymb}
\usepackage{wrapfig}
\usepackage{graphicx}
\usepackage{colortbl} 
\usepackage{subcaption}
\usepackage{caption}
\usepackage{threeparttable}
\usepackage{amsthm}
\usepackage{multirow}

\newtheorem{assumption}{Assumption}
\newtheorem{proposition}{Proposition}
\newtheorem{remark}{Remark}

\newcommand{\pmval}[1]{\scriptsize{$\pm$#1}}

\newcommand{\synx}{\hat{\boldsymbol{x}}}

\title{Spectral Gradient Surgery for\\ Domain-Generalizable Dataset Distillation}

\author{%
  Minyoung Oh, Najeong Chae, and Jae-Young Sim\thanks{corresponding author.} \\
  Graduate School of Artificial Intelligence\\
  Ulsan National Institute of Science and Technology~(UNIST)\\
  Ulsan, Republic of Korea \\
  \texttt{\{mmyy2513, njchae, jysim\}@unist.ac.kr} \\
}

\begin{document}

\maketitle

\begin{abstract}
Dataset Distillation (DD) synthesizes a compact synthetic dataset that preserves the training utility of a full dataset. However, its standard formulation assumes that test data follow the same distribution as training data, an assumption that rarely holds in practice. A straightforward extension—applying post-hoc Domain Generalization (DG) techniques to distilled data—is ill-suited because existing DG methods rely on the natural diversity of real datasets, which compact synthetic sets inherently lack, while also incurring substantial augmentation overhead that conflicts with the efficiency objective of DD.
To address this limitation, we introduce Domain Generalizable Dataset Distillation (DGDD), a new problem setting that explicitly targets out-of-distribution (OOD) generalization of distilled datasets. We study this problem through a widely adopted DD baseline of Distribution Matching (DM). We attribute the OOD vulnerability of DM to the entanglement of class-discriminative and domain-specific information within the compressed synthetic set, and propose Spectral Gradient Surgery (SGS) to disentangle the two.
The key insight of SGS is that cross-domain agreement among domain-wise gradients in the spectral domain reveals which gradient components are shared across source domains—and are therefore class-discriminative—and which are domain-specific. Based on this observation, SGS augments the standard DM update with two complementary gradients: one that reinforces cross-domain shared components and another that explicitly promotes diversity within the distilled dataset. Extensive experiments on diverse-scale benchmarks demonstrate that SGS substantially improves OOD generalization while remaining plug-and-play compatible with existing DM methods.
\end{abstract}

\section{Introduction}

In recent years, deep neural networks have achieved remarkable progress across diverse fields, driven by the availability of large-scale training datasets. 
However, the increasing scale of training data introduces substantial storage and computational costs, imposing significant bottlenecks on practitioners. 
To address this, Dataset Distillation (DD)~\citep{dd} has attracted significant attention as a promising solution.
DD synthesizes a compact dataset such that a model trained on it achieves performance comparable to one trained on the full dataset. This enables a range of applications, including data storage reduction and efficient training under resource constraints, thereby accelerating downstream model development.

DD methods typically update the synthetic dataset by matching the training dynamics between the synthetic and real datasets so that the synthetic dataset can serve as a faithful proxy for the full dataset. Specifically, DD methods are categorized according to the type of training dynamics they match: (1) Distribution Matching (DM)~\citep{dm}, which aligns feature distributions; (2) Gradient Matching (GM)~\citep{gm}, which aligns per-step gradients; and (3) Trajectory Matching (TM)~\citep{mtt}, which aligns parameter trajectories over training. 
However, existing DD research predominantly assumes that the training and test data are independent and identically distributed (i.i.d.), which rarely holds in practice.
Under domain shift, models trained on distilled data often exhibit significant performance degradation, limiting the practical applicability of DD in real-world scenarios.
We therefore argue that improving the robustness of distilled datasets under domain shift is an important research direction. 
To this end, we investigate this problem through the lens of DM, which has been widely adopted as a baseline due to its strong performance and computational efficiency.

\begin{wrapfigure}{r}{0.45\linewidth}
  \begin{center}
    \vspace{-15pt} 
    \includegraphics[width=\linewidth]{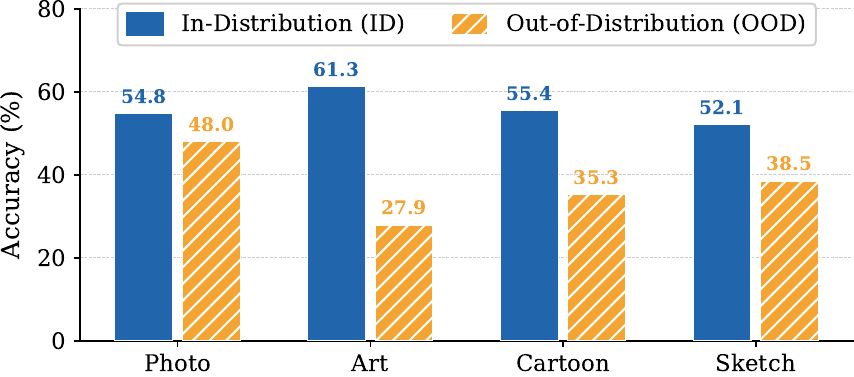}
    \caption{Comparison of ID and OOD performance for each target domain on PACS.}
    \label{fig:ood_drop}
    \vspace{3mm}
    \includegraphics[width=\linewidth]{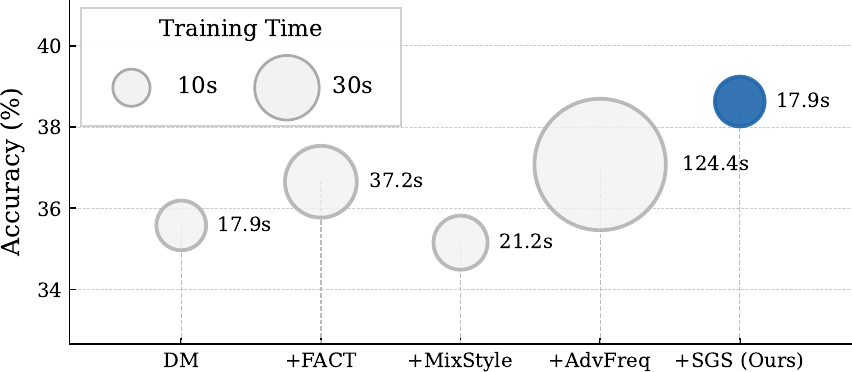}
    \caption{Comparison of average OOD performance and downstream training time on PACS when applying DG methods post-hoc to DM-distilled datasets.}
    \label{fig:posthoc}
    \vspace{-10pt}
  \end{center}
\end{wrapfigure}

We assess the robustness of DM-distilled dataset with 10 images-per-class against domain shifts on the PACS~\citep{pacs} dataset under two evaluation settings: \textbf{in-distribution (ID)} (evaluated on source domains) and \textbf{out-of-distribution (OOD)} (evaluated on an unseen target domain via leave-one-domain-out). 
As shown in Figure~\ref{fig:ood_drop}, existing DM methods are highly vulnerable to domain shifts and suffer from severe OOD performance degradation. Although applying Domain Generalization (DG) methods~\citep{fact, mixstyle} during downstream training is a straightforward remedy, Figure~\ref{fig:posthoc} shows that such post-hoc approaches provide only marginal gains and can even harm performance. This is because distilled datasets visually differ from real images, violating the assumptions underlying existing DG methods. Moreover, their heavy augmentation strategies significantly increase training time, contradicting DD's core objective of rapid training. These findings indicate that naively applying existing DG techniques is insufficient. Motivated by this, we propose a new problem setting termed \textbf{Domain Generalizable Dataset Distillation~(DGDD)}.

We attribute the OOD performance degradation of existing DM methods to the unintended distillation of domain-specific information alongside class-discriminative information, as these methods compress the entire training dataset into a compact synthetic one without distinguishing between the two. 
To mitigate this, we propose \textbf{Spectral Gradient Surgery (SGS)}, which augments the standard DM update with two additional gradient signals derived from per-domain decomposition in the spectral domain.
Specifically, SGS computes domain-wise DM gradients, transforms them into the spectral domain, and measures cross-domain agreement at each frequency component via the resultant length. 
Components with high agreement are aggregated into a class-discriminative signal that reinforces patterns shared across source domains, while per-domain deviations from this consensus form a domain-specific signal that preserves distinct domain characteristics, encouraging diversity within the distilled dataset.
By design, SGS is plug-and-play compatible with existing DM methods, requiring no modification to their core objectives.
Extensive experiments demonstrate that SGS substantially improves the OOD generalization of distilled datasets across diverse benchmarks.

Our contributions can be summarized as follows:
\begin{itemize}
    \item To the best of our knowledge, we first introduce a novel problem, Domain Generalizable Dataset Distillation~(DGDD). Unlike conventional DD, which primarily focuses on preserving ID performance, DGDD explicitly targets robustness to domain shift and strong OOD generalization.
    \item We propose Spectral Gradient Surgery~(SGS), a plug-and-play extension to existing DM methods that augments standard DM updates with two spectral-domain gradient signals: a class-discriminative signal that captures domain-invariant patterns shared across source domains, and a domain-specific signal that preserves distinctive domain characteristics to encourage diversity among synthetic samples.
\end{itemize}

\section{Related Work}
\subsection{Dataset Distillation}
Existing dataset distillation methods are typically categorized into three groups: Gradient Matching~(GM)~\citep{gm, dsa}, which aligns per-step gradients between real and synthetic data; Trajectory Matching~(TM)~\citep{mtt, tesla}, which aligns full parameter trajectories over training; and Distribution Matching~(DM)~\citep{dm}, which is the primary focus of this paper. 
While GM and TM achieve strong performance, their reliance on costly bi-level optimization limits scalability to large datasets. 
DM methods address this by directly aligning feature distributions between real and synthetic data using randomly initialized encoders, offering a significantly more efficient alternative.
Subsequent works have improved DM along several directions. 
\citet{datadam} proposes spatial attention matching to better capture discriminative features across multiple network layers. 
\citet{m3d} extends distribution matching beyond first-order moments by minimizing Maximum Mean Discrepancy in a reproducing kernel Hilbert space, enabling higher-order moment alignment. 
\citet{ncfm} reformulates distribution matching as a minmax optimization over neural characteristic functions, aligning both their phase and amplitude in the complex plane. 
\citet{hdd} introduces hyperbolic geometry into dataset distillation, embedding feature distributions into the Lorentz model to naturally capture hierarchical sample structure.
However, all of these methods aim exclusively to match the feature distribution of the source dataset, optimizing solely for in-distribution performance. 
While a recent work~\citep{asu} introduces a framework that distills from streaming multi-domain datasets, it still restricts its scope to seen domains and does not consider generalization to unseen target domains. 
Consequently, the robustness of distilled datasets against domain shift—a frequent challenge in real-world deployment—remains largely unexplored. 
In this paper, we address this critical gap by explicitly targeting out-of-distribution generalization.

\subsection{Domain Generalization}

Domain generalization (DG) aims to mitigate performance degradation under distribution shifts, a critical challenge in deploying AI models. 
Existing DG methods can be broadly classified into two main approaches. 
The first approach enhances training diversity through data augmentation. 
For instance, \citet{fact} mixes the amplitude spectrum of source images while preserving the phase. 
\citet{mixstyle} probabilistically mixes channel-wise feature statistics across instances to synthesize novel styles. 
\citet{advst} adversarially optimizes parameterized augmentations to generate challenging samples, while \citet{advfreq} perturbs the amplitude spectrum to reduce reliance on shortcut frequency bands.
The second approach enforces invariance via cross-domain gradient alignment during training. 
\citet{irm} enforces a shared optimal classifier across domains, 
\citet{fish} aligns per-domain gradient means, and \citet{fishr} matches their variances. 
These approaches operate on parameter-space gradients, whereas our proposed SGS applies cross-domain agreement to the image-space gradients of synthetic data. 
This critical distinction enables a natural spectral decomposition on the 2D Fourier basis, which is neither available nor meaningful in parameter space.
While applying existing DG methods post-hoc to distilled datasets might seem straightforward, it is fundamentally insufficient for two reasons. 
First, these methods assume access to diverse data drawn from a natural image distribution—an assumption violated by compact synthetic datasets. 
Second, their reliance on heavy augmentation pipelines introduces substantial computational overhead, undermining the core advantage of dataset distillation: efficient downstream training. 
Therefore, a dedicated approach that integrates domain generalization directly into the distillation process is strictly required.

\section{Methodology}\label{sec:method}
\subsection{Problem Formulation}
Let $\mathcal{D} = \bigcup_{s=1}^{S} \mathcal{D}^s$ be the union of $S$ source datasets representing distinct domains over a shared class label space $\mathcal{Y}$, with $\mathbb{P}^{\mathcal{D}}$ denoting its mixture distribution. We define $\mathcal{T}$ as the test dataset drawn from a target distribution $\mathbb{P}^{\mathcal{T}}$.
Given $\mathcal{D}$, dataset distillation synthesizes a compact synthetic dataset $\hat{\mathcal{D}} = \{(\hat{\boldsymbol{x}}_j, \hat{y}_j)\}$ that serves as a faithful proxy for $\mathcal{D}$.
The synthetic dataset $\hat{\mathcal{D}}$ is typically much smaller than the full dataset $\mathcal{D}$ ($|\hat{\mathcal{D}}| \ll |\mathcal{D}|$), where $|\hat{\mathcal{D}}| = \text{IPC} \times C$, with IPC denoting the number of images per class and $C$ the number of classes in $\mathcal{Y}$.
The goal of conventional DD is to find $\hat{\mathcal{D}}$ that minimizes the performance gap between networks trained on $\hat{\mathcal{D}}$ and on $\mathcal{D}$, measured by a loss function $\mathcal{L}$ on $\mathcal{T}$ based on the i.i.d.\ assumption:
\begin{equation}
    \hat{\mathcal{D}}^* = \arg\min_{\hat{\mathcal{D}}} \mathbb{E}_{(\boldsymbol{x}, y) \sim \mathcal{T}} \left[\mathcal{L}(\phi_{\hat{\mathcal{D}}}(\boldsymbol{x}, y)) - \mathcal{L}(\phi_{\mathcal{D}}(\boldsymbol{x}, y))\right],\quad \text{s.t.}\quad\mathbb{P}^{\mathcal{T}} = \mathbb{P}^{\mathcal{D}},
    \label{eq:dd}
\end{equation}
where $\phi_{\hat{\mathcal{D}}}$ and $\phi_{\mathcal{D}}$ denote networks trained on $\hat{\mathcal{D}}$ and $\mathcal{D}$, respectively.

However, the strict i.i.d.\ assumption rarely holds in the presence of domain shift. To address this limitation, we introduce a more general problem setting of \textbf{Domain Generalizable Dataset Distillation (DGDD)}:
\begin{equation}
	\hat{\mathcal{D}}^* = \arg\min_{\hat{\mathcal{D}}} \mathbb{E}_{(\boldsymbol{x}, y) \sim \mathcal{T}} \left[\mathcal{L}(\phi_{\hat{\mathcal{D}}}(\boldsymbol{x}, y)) - \mathcal{L}(\phi_{\mathcal{D}}(\boldsymbol{x}, y))\right],\quad \text{s.t.}\quad \text{supp}(\mathbb{P}^{\mathcal{T}}) \setminus \text{supp}(\mathbb{P}^{\mathcal{D}}) \neq \varnothing.
    \label{eq:dgdd}
\end{equation}
Unlike conventional DD settings that rely on the strict i.i.d.\ assumption, DGDD explicitly targets robustness to domain shift by evaluating on a target distribution $\mathbb{P}^{\mathcal{T}}$ that is not fully covered by the source distribution $\mathbb{P}^{\mathcal{D}}$. By explicitly modeling unseen domain characteristics, this formulation unifies ID and OOD evaluation under a single objective: a proxy model trained on the distilled dataset $\hat{\mathcal{D}}$ should generalize robustly to both seen source domains and unseen target domains.

\subsection{Dissecting Distribution Matching}\label{sec:dissect}
\begin{wrapfigure}[13]{r}{0.3\linewidth} 
  \vspace{-8mm}
  \centering
    \includegraphics[width=0.95\linewidth]{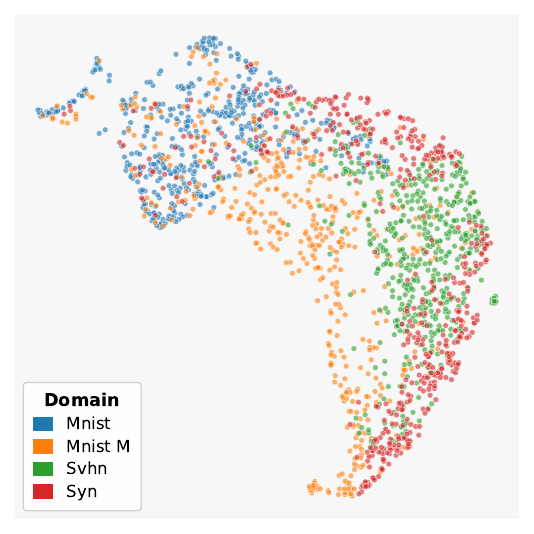}
    \vspace{-3mm}
    \caption{UMAP visualization of feature statistics on the Digits-DG dataset. Each color indicates a distinct domain.}
    \label{fig:umap_domain}
\end{wrapfigure}

Distribution Matching~(DM)~\citep{dm} updates a synthetic dataset $\hat{\mathcal{D}}$ by minimizing the discrepancy between the feature distributions of synthetic and real data, computed per class:
\begin{equation}
    \mathcal{L}_{\mathrm{DM}} = \sum_{c=1}^{C} \left\| \mathbb{E}_{\hat{\boldsymbol{x}} \in \hat{\mathcal{D}}_c} [\psi(\hat{\boldsymbol{x}})] - 
    \mathbb{E}_{\boldsymbol{x} \in \mathcal{D}_c} [\psi(\boldsymbol{x})] \right\|^2,
    \label{eq:dm}
\end{equation}
where $\psi$ is a feature extractor, and $\hat{\mathcal{D}}_c$ and $\mathcal{D}_c$ are the subsets of $\hat{\mathcal{D}}$ and $\mathcal{D}$, respectively, belonging to the $c$-th class. 
Based on Eq.~\ref{eq:dm}, the $i$-th synthetic sample $\synx_i$ is updated via the gradient induced by $\mathcal{L}_\mathrm{DM}$ with a learning rate $\eta$:
\begin{equation}
    \synx_i \leftarrow \synx_i - \eta \cdot \boldsymbol{g}_i,\quad \text{where}\quad \boldsymbol{g}_i=\frac{\partial\mathcal{L}_\mathrm{DM}}{\partial\synx_i}.
    \label{eq:dm_update}
\end{equation}

Crucially, the standard gradient $\boldsymbol{g}_i$ inherently entangles class-discriminative signals with domain-specific information from the source datasets. Figure~\ref{fig:umap_domain} provides empirical evidence for this entanglement. Visualizing the feature statistics (i.e., channel-wise mean and standard deviation) of Digits-DG samples via UMAP~\citep{umap} reveals clear domain-wise clustering. This confirms that the extracted features encode substantial domain-specific characteristics alongside class semantics. 

This observation motivates us to decompose the gradient $\boldsymbol{g}_i$ into two distinct components:
\begin{equation}
    \boldsymbol{g}_i = \boldsymbol{g}_i^\text{class} + \boldsymbol{g}_i^{\text{domain}},
    \label{eq:decompose}
\end{equation}
where $\boldsymbol{g}_i^\text{class}$ represents the class-discriminative information shared across domains, and $\boldsymbol{g}_i^{\text{domain}}$ captures the domain-specific nuances particular to the source datasets. Consequently, we leverage this conceptual decomposition to enhance OOD generalization—explicitly reinforcing the domain-invariant class signals while selectively controlling the influence of domain-specific updates.

\subsection{Spectral Gradient Surgery}
Based on the analysis in Section~\ref{sec:dissect}, we propose \textbf{Spectral Gradient Surgery~(SGS)} to operationalize the decomposition in Eq.~\ref{eq:decompose}. Our core intuition is that cross-domain gradient agreement serves as a highly reliable proxy for distinguishing between spectral components that are shared across domains—and thus class-discriminative—and those that are strictly domain-specific. Building upon this proxy, SGS augments the standard DM update by explicitly steering it with two complementary signals: a \textit{class-discriminative} signal that reinforces the spectral components exhibiting strong cross-domain consensus, and a \textit{domain-specific} signal that preserves distinct domain variations to promote visual diversity within the synthetic dataset.

\paragraph{Domain-wise Gradient Computation.}
Given source datasets $\{\mathcal{D}^s\}_{s=1}^S$ from $S$ domains, we compute a domain-wise gradient for each source domain $\mathcal{D}^s$:
\begin{equation}
    \boldsymbol{g}_i^s = \frac{\partial \mathcal{L}_{\mathrm{DM}}^s}
    {\partial \hat{\boldsymbol{x}}_i},
    \quad \mathcal{L}_{\mathrm{DM}}^s = \sum_{c=1}^{C} \left\| 
    \mathbb{E}_{\hat{\boldsymbol{x}} \in \hat{\mathcal{D}}_c} 
    [\psi(\hat{\boldsymbol{x}})] - 
    \mathbb{E}_{\boldsymbol{x} \in \mathcal{D}^s_c} 
    [\psi(\boldsymbol{x})] \right\|^2,
    \label{eq:dm_domain}
\end{equation}
where $\mathcal{D}^s_c$ denotes the subset of $\mathcal{D}^s$ belonging to the $c$-th class. Intuitively, $\boldsymbol{g}_i^s$ encodes the direction in which $\hat{\boldsymbol{x}}_i$ should move to better represent domain $\mathcal{D}^s$, and thus contains both class-discriminative and domain-specific signals.

\paragraph{Spectral Gradient Decomposition.} 
Note that $\boldsymbol{g}_i^s$ represents an image-space offset from $\hat{\boldsymbol{x}}_i$. Before aggregating the domain-wise gradients, we deliberately operate in the spectral domain rather than the pixel space for two crucial reasons. 
First, pixel-wise cross-domain agreement is highly susceptible to spatial misalignment; computing gradient agreement at a specific spatial coordinate is semantically meaningless. 
Second, a single frequency component represents a global spatial pattern over the entire image. Measuring cross-domain agreement in the spectral domain reflects whether a global, semantically meaningful pattern is shared across domains, naturally preserving spatial coherence.
Specifically, we transform each source-domain gradient $\boldsymbol{g}_i^s$ into the spectral domain signal via the 2D Fast Fourier Transform (FFT) $\mathcal{F}$:
\begin{equation}
    G_i^s = \mathcal{F}(\boldsymbol{g}_i^s).
    \label{eq:fft}
\end{equation}
We then compute the complex mean of $\{G_i^s\}_{s=1}^S$ across all source domains given by 
\begin{equation}
    \bar{G}_i = \frac{1}{S} \sum_{s=1}^{S} G_i^s,
\end{equation}
which captures the gradient patterns consistently shared across all source domains.

However, not all frequency components contribute equally to this consensus: some exhibit strong cross-domain agreement, while others are dominated by domain-specific noise. To quantify the reliability of this consensus, we compute the element-wise resultant length from circular statistics:
\begin{equation}
    r_i = \frac{\left| \sum_{s=1}^{S} G_i^s \right|}{\sum_{s=1}^{S} |G_i^s| + \epsilon} \in [0, 1],
\end{equation}
where $|\cdot|$ denotes the element-wise complex magnitude and $\epsilon$ is a small constant for numerical stability. 
Here, $r_i \rightarrow 1$ indicates that the domain-wise gradients are aligned at a given frequency component, corresponding to a strong class-discriminative signal, whereas $r_i \rightarrow 0$ indicates conflicting gradient directions, corresponding to domain-specific noise.
Weighting $\bar{G}_i$ by $r_i$ selectively preserves domain-invariant consensus while attenuating domain-specific noise, effectively realizing the decomposition.
\begin{equation}
\begin{aligned}
    \boldsymbol{g}_i^\text{class} &= \mathcal{F}^{-1}
    (\bar{G}_i \cdot r_i), \\    \boldsymbol{g}_i^{\text{domain}, s} &= \mathcal{F}^{-1}
    (G_i^s - \bar{G}_i),
\end{aligned}
\label{eq:sgs_decompose}
\end{equation}
where $\boldsymbol{g}_i^\text{class}$ preserves components with reliable consensus~($r_i \to 1$) while attenuating conflicting ones~($r_i \to 0$), and $\boldsymbol{g}_i^{\text{domain},s}$ captures the deviation of domain $s$ from the consensus, encoding domain-specific characteristics of $\mathcal{D}^s$.
We provide a formal theoretical justification in Appendix~\ref{app:theory}.

\paragraph{Adaptive Synthetic Data Optimization.} 
Three complementary gradient signals are used to update the model.
First, $\boldsymbol{g}_i$ — the standard DM gradient computed against the union of all source domains — provides the primary update direction, preserving the strong baseline performance of DM.
In addition, the class-discriminative signal $\boldsymbol{g}^{\text{class}}_i$ reinforces gradient patterns that are consistent across source domains, guiding the synthetic data toward domain-invariant representations. 
The domain-specific signal $\boldsymbol{g}^{\text{domain}, s(i)}_i$ preserves distinct domain characteristics, where $s(i) \in \{1, \dots, S\}$ denotes the pre-assigned source domain index for the $i$-th synthetic sample $\hat{\boldsymbol{x}}_i$. 
To encourage diversity within the distilled dataset, this assignment is uniformly distributed such that every source domain is allocated an equal number of synthetic samples.
Consequently, the synthetic data sample is optimized as
\begin{equation}
    \hat{\boldsymbol{x}}_i \leftarrow \hat{\boldsymbol{x}}_i - 
    \eta \cdot \left(
    \boldsymbol{g}_i +
    \lambda_c \cdot \boldsymbol{g}_i^\text{class} + 
    \lambda_d \cdot \boldsymbol{g}_i^{\text{domain}, s(i)}
    \right),
    \label{eq:sgs_final}
\end{equation}
where $\lambda_c$ and $\lambda_d$ control the strength of each signal. 
Note that setting $\lambda_c=0$ and $\lambda_d=0$ recovers the standard DM update, demonstrating that SGS generalizes DM.

\section{Experiments}\label{sec:exp}

\subsection{Experimental Setup}
\noindent\textbf{Datasets.}\quad
To verify the effectiveness of the proposed method, we employ datasets with diverse scales: Digits-DG~\citep{digitsDG}, PACS~\citep{pacs}, and CORe50-Hard~\citep{core50}. 
Digits-DG is a small-scale dataset with 10 classes at $32 \times 32$ resolution, comprising four digit domains: 
MNIST~(M)~\citep{mnist}, MNIST-M~(N)~\citep{syn}, SVHN~(V)~\citep{svhn}, and SYN~(Y)~\citep{syn}.
PACS is a domain generalization benchmark at medium scale ($64 \times 64$ resolution), consisting of four domains: Photo~(P), Art painting~(A), Cartoon~(C), and Sketch~(S), across 7 classes. 
For large-scale evaluation, we construct CORe50-Hard from CORe50 by selecting 10 classes four sessions (s5, s7, s9, s10) that span the most diverse visual conditions, containing real-world images at $128 \times 128$ resolution.

\begin{table}[t]
    \caption{Multi-domain generalization performance on the Digits-DG, PACS, and CORe50-Hard datasets with IPC=10 and IPC=20.}
    \centering
    \resizebox{0.97\linewidth}{!}{
    \begin{tabular}{c|ccccc|ccccc}
    \toprule
    Dataset & \multicolumn{10}{c}{Digits-DG} \\ 
    IPC & \multicolumn{5}{c|}{10} & \multicolumn{5}{c}{20} \\
    Target & M & N & V & Y & Avg. & M & N & V & Y & Avg. \\ \midrule
    \multicolumn{1}{l|}{Whole}
    & 97.8\pmval0.1 & 66.2\pmval0.7& 75.5\pmval0.2& 92.7\pmval0.2 & 83.1\pmval0.3
    & 97.8\pmval0.1 & 66.2\pmval0.7& 75.5\pmval0.2& 92.7\pmval0.2 & 83.1\pmval0.3\\
    \midrule
    
    \multicolumn{1}{l|}{Random} 
    & 42.6\pmval2.2 & 34.2\pmval1.6& 29.0\pmval0.9& 39.2\pmval0.4 & 36.3\pmval1.3
    & 65.9\pmval0.7 & 38.5\pmval1.1 & 38.5\pmval0.3 & 52.6\pmval0.5 & 48.9\pmval0.7\\
    
    \multicolumn{1}{l|}{Herding} 
    & 69.9\pmval1.5 & 38.5\pmval0.9& 39.2\pmval0.6& 51.5\pmval0.7 & 49.8\pmval0.9
    & 78.3\pmval2.0 & 42.4\pmval0.7 & 44.3\pmval0.9 & 62.0\pmval0.7 & 56.8\pmval1.1\\
    \midrule
    
    \multicolumn{1}{l|}{DataDAM} 
    & 80.0\pmval1.6 & 47.4\pmval1.1 & 33.5\pmval0.7 & 46.7\pmval1.8 & 51.9\pmval1.3
    & 88.1\pmval0.7 & 52.0\pmval1.2 & 40.6\pmval1.6 & 62.0\pmval0.6 & 60.7\pmval1.0 \\
    
    \multicolumn{1}{l|}{M3D} 
    & \underline{88.1}\pmval0.5 & 41.5\pmval0.1 & 46.8\pmval0.2 & 62.3\pmval1.0 & 59.7\pmval0.5
    & \underline{90.1}\pmval1.0 & 46.6\pmval0.9 & 56.4\pmval0.3 & 71.2\pmval1.0 & 66.1\pmval0.8 \\ 
    
    \multicolumn{1}{l|}{DM} 
    & 85.3\pmval0.7 & 51.4\pmval1.5 & 53.6\pmval0.9 & 69.9\pmval0.9 & 65.1\pmval1.0
    & 88.4\pmval0.8 & 53.1\pmval1.0 & 56.4\pmval1.1 & \underline{73.7}\pmval0.2 & 67.9\pmval0.8\\
    
    \rowcolor[gray]{0.9}
    \multicolumn{1}{l|}{\quad + SGS} 
    & \textcolor{red}{\textbf{90.3}}\pmval0.4 & \textcolor{red}{\textbf{54.1}}\pmval0.9 & \textcolor{red}{\textbf{55.3}}\pmval0.4 & \textcolor{red}{\textbf{71.0}}\pmval0.8 & \textcolor{red}{\textbf{67.7}}\pmval0.6
    & \textcolor{red}{\textbf{90.7}}\pmval0.5 & \textcolor{red}{54.3}\pmval1.2 & \textcolor{red}{\textbf{58.0}}\pmval0.5 & \textcolor{red}{\textbf{74.3}}\pmval0.8 & \textcolor{red}{\textbf{69.3}}\pmval0.8\\
    
    \multicolumn{1}{l|}{HDD-DM} 
    & 86.3\pmval0.6 & {52.6}\pmval1.1 & 53.5\pmval0.8 & 68.4\pmval1.0 & 65.2\pmval0.9
    & 89.5\pmval0.4 & \underline{54.4}\pmval1.7 & 56.7\pmval1.1 & 73.2\pmval0.3 & 68.5\pmval0.9\\
    
    \rowcolor[gray]{0.9}
    \multicolumn{1}{l|}{\quad + SGS} 
    & \textcolor{red}{87.9}\pmval1.2 & \textcolor{red}{\underline{53.7}}\pmval2.2 & \textcolor{red}{\underline{54.1}}\pmval0.9 & \textcolor{red}{\underline{70.0}}\pmval0.4 & \textcolor{red}{\underline{66.4}}\pmval1.2
    & \textcolor{red}{\textbf{90.7}}\pmval0.3 & \textcolor{red}{\textbf{54.9}}\pmval1.3 & \textcolor{red}{\underline{57.5}}\pmval1.0 & \textcolor{red}{73.4}\pmval0.7 & \textcolor{red}{\underline{69.1}}\pmval0.8\\
    \bottomrule
    \multicolumn{11}{c}{}\\
    
    \toprule
    Dataset & \multicolumn{10}{c}{PACS} \\
    IPC & \multicolumn{5}{c|}{10} & \multicolumn{5}{c}{20} \\
    Target & P & A & C & S & Avg. & P & A & C & S & Avg. \\ \midrule
    \multicolumn{1}{l|}{Whole}
    & 59.1\pmval1.4 & 36.7\pmval0.4 & 52.9\pmval1.2 & 50.2\pmval1.0 & 49.7\pmval1.0
    & 59.1\pmval1.4 & 36.7\pmval0.4 & 52.9\pmval1.2 & 50.2\pmval1.0 & 49.7\pmval1.0 \\
    \midrule
    
    \multicolumn{1}{l|}{Random} 
    & 38.3\pmval1.0 & 22.5\pmval0.1 & 25.4\pmval1.1 & 24.9\pmval2.3 & 27.8\pmval1.1
    & 45.6\pmval1.8 & 23.3\pmval0.8 & 27.7\pmval0.7 & 28.0\pmval0.7 & 31.2\pmval1.0 \\
    
    \multicolumn{1}{l|}{Herding} 
    & 37.9\pmval3.3 & 25.0\pmval0.9 & 30.2\pmval2.1 & 26.4\pmval4.0 & 29.9\pmval2.6
    & 43.7\pmval1.8 & 26.5\pmval0.7 & 32.2\pmval1.6 & 24.8\pmval2.6 & 31.8\pmval1.7 \\
    \midrule
    
    \multicolumn{1}{l|}{DataDAM} 
    & 44.3\pmval1.4 & 27.6\pmval0.9 & 32.5\pmval0.9 & 27.2\pmval4.8 & 32.9\pmval2.0
    & 49.6\pmval1.3 & 29.7\pmval0.4 & 38.3\pmval0.5 & 36.3\pmval2.5 & 38.5\pmval1.2 \\  
    
    \multicolumn{1}{l|}{M3D} 
    & 49.0\pmval0.9 & 27.1\pmval0.6 & \textbf{35.9}\pmval1.4 & 25.5\pmval2.7 & 34.4\pmval1.4
    & 48.6\pmval0.5 & \textbf{30.3}\pmval1.0 & \underline{38.6}\pmval0.7 & \underline{39.3}\pmval2.7 & 39.2\pmval1.2 \\ 
    
    \multicolumn{1}{l|}{DM} 
    & 48.0\pmval0.7 & 27.9\pmval0.5 & 33.8\pmval1.5 & 32.6\pmval2.2 & 35.6\pmval1.2
    & 49.2\pmval0.4 & 28.8\pmval1.0 & 35.4\pmval2.2 & 38.1\pmval2.7 & 37.9\pmval1.6 \\   
    
    \rowcolor[gray]{0.9}
    \multicolumn{1}{l|}{\quad + SGS} 
    & \textcolor{red}{\textbf{53.8}}\pmval0.4 & \textcolor{red}{\underline{28.4}}\pmval1.1 & \textcolor{red}{\underline{35.1}}\pmval1.5 & \textcolor{red}{\textbf{37.2}}\pmval1.8 & \textcolor{red}{\textbf{38.6}}\pmval1.2
    & \textcolor{red}{\textbf{53.2}}\pmval1.0 & \textcolor{red}{29.1}\pmval0.6 & \textcolor{red}{37.7}\pmval2.3 & \textcolor{red}{\textbf{39.8}}\pmval2.5 & \textcolor{red}{\textbf{40.0}}\pmval1.6 \\ 
    
    \multicolumn{1}{l|}{HDD-DM} 
    & 47.8\pmval0.4 & \underline{28.4}\pmval0.7 & 34.1\pmval1.8 & 29.2\pmval3.0 & 34.9\pmval1.5
    & 48.9\pmval1.2 & 29.5\pmval1.2 & 37.1\pmval1.7 & 32.5\pmval1.7 & 37.0\pmval1.5 \\
    \rowcolor[gray]{0.9}
    
    \multicolumn{1}{l|}{\quad + SGS} 
    & \textcolor{red}{\underline{50.1}}\pmval2.2 & \textcolor{red}{\textbf{29.1}}\pmval0.9 & \textcolor{red}{\textbf{35.9}}\pmval1.1 & \textcolor{red}{\underline{33.4}}\pmval1.9 & \textcolor{red}{\underline{37.1}}\pmval1.5
    & \textcolor{red}{\underline{51.8}}\pmval0.7 & \textcolor{red}{\underline{29.9}}\pmval0.8 & \textcolor{red}{\textbf{39.3}}\pmval1.2 & \textcolor{red}{37.5}\pmval2.5  & \textcolor{red}{\underline{39.6}}\pmval1.3 \\
    \bottomrule
    \multicolumn{11}{c}{}\\
    
    \toprule
    Dataset & \multicolumn{10}{c}{CORe50-Hard} \\
    IPC & \multicolumn{5}{c|}{10} & \multicolumn{5}{c}{20} \\
    Target & s5 & s7 & s9 & s10 & Avg. & s5 & s7 & s9 & s10 & Avg. \\ \midrule
    \multicolumn{1}{l|}{Whole}
    & 54.4\pmval2.0 & 46.5\pmval1.7 & 50.6\pmval1.5 & 26.7\pmval1.2 & 44.6\pmval1.6
    & 54.4\pmval2.0 & 46.5\pmval1.7 & 50.6\pmval1.5 & 26.7\pmval1.2 & 44.6\pmval1.6\\
    \midrule
    
    \multicolumn{1}{l|}{Random} 
    & 23.0\pmval0.4 & 19.8\pmval0.3 & 20.7\pmval0.9 & 16.8\pmval0.4 & 20.1\pmval0.5
    & 24.2\pmval0.6 & 26.0\pmval1.1 & 24.4\pmval1.0 & 20.0\pmval0.9 & 23.7\pmval0.9\\
    
    \multicolumn{1}{l|}{Herding} 
    & 24.7\pmval1.4 & 20.9\pmval0.8 & 24.6\pmval1.9 & 19.9\pmval1.1 & 22.5\pmval1.3
    & 28.1\pmval1.6 & 27.1\pmval0.6 & 28.9\pmval0.5 & 19.9\pmval0.7 & 26.0\pmval0.9\\
    \midrule
    
    \multicolumn{1}{l|}{DataDAM} 
    & \underline{29.3}\pmval2.5 & \underline{31.7}\pmval0.7 & 29.1\pmval1.0 & 20.0\pmval1.0 & \underline{27.5}\pmval1.3
    & \underline{31.1}\pmval1.0 & 28.9\pmval1.0 & 28.0\pmval1.0 & 20.2\pmval1.0 & 27.1\pmval1.0 \\  
    
    \multicolumn{1}{l|}{M3D} 
    & 20.4\pmval0.8 & 30.0\pmval1.4 & 26.0\pmval0.7 & 18.3\pmval0.6 & 23.7\pmval0.9
    & 22.0\pmval1.2 & 29.1\pmval2.0 & 28.5\pmval1.4 & \textbf{21.6}\pmval0.8 & 25.3\pmval1.4 \\ 
    
    \multicolumn{1}{l|}{DM} 
    & 28.4\pmval1.7 & 31.4\pmval2.0 & \underline{30.3}\pmval2.2 & 19.3\pmval1.4 & 27.4\pmval1.8
    & 30.5\pmval2.2 & \underline{32.5}\pmval2.1 & \underline{31.9}\pmval0.9 & 20.5\pmval1.1 & \underline{28.9}\pmval1.6\\
    
    \rowcolor[gray]{0.9}
    \multicolumn{1}{l|}{\quad + SGS} 
    & \textcolor{red}{\textbf{29.8}}\pmval1.9 & \textcolor{red}{\textbf{32.6}}\pmval2.1 & \textcolor{red}{\textbf{32.3}}\pmval1.7 & \textcolor{red}{\underline{21.0}}\pmval1.1 & \textcolor{red}{\textbf{28.9}}\pmval1.7
    & \textcolor{red}{\textbf{32.0}}\pmval2.0 & \textcolor{red}{\textbf{34.7}}\pmval0.6 & \textcolor{red}{\textbf{34.0}}\pmval1.1 & \textcolor{red}{\underline{20.9}}\pmval1.0 & \textcolor{red}{\textbf{30.4}}\pmval1.2\\
    
    \multicolumn{1}{l|}{HDD-DM} 
    & {25.6}\pmval1.4 & 25.4\pmval2.1 & 24.6\pmval0.9 & 18.7\pmval1.3 & 23.6\pmval1.4
    & 27.7\pmval2.1 & 30.9\pmval0.9 & 29.3\pmval0.6 & 19.0\pmval1.4 & 26.7\pmval1.3\\
    
    \rowcolor[gray]{0.9}
    \multicolumn{1}{l|}{\quad + SGS} 
    & \textcolor{red}{29.0}\pmval1.0 & \textcolor{red}{28.9}\pmval1.6 & \textcolor{red}{28.6}\pmval0.9 & \textcolor{red}{\textbf{21.3}}\pmval1.8 & \textcolor{red}{27.0}\pmval1.3
    & \textcolor{red}{30.6}\pmval2.0 & \textcolor{red}{32.0}\pmval1.4 & \textcolor{red}{29.4}\pmval1.4 & \textcolor{red}{20.1}\pmval1.0 & \textcolor{red}{28.0}\pmval1.5\\
    \bottomrule
    \end{tabular}
    }
    \label{tab:mdg}
\end{table}

\noindent\textbf{Implementation Details.}\quad
We implement all experiments in PyTorch. 
For the feature extractor $\psi$, we employ ConvNets with 3, 4, and 5 layers for Digits-DG, PACS, and CORe50-Hard, respectively, scaling the network depth according to the dataset resolution following prior works~\citep{hdd, m3d}. 
For distillation, we initialize the synthetic samples using real images uniformly sampled across the source domains—ensuring an equal contribution from each domain—and optimize them using Stochastic Gradient Descent (SGD) for 20,000 iterations. 
Following standard DM practice, we apply differentiable augmentation~(DSA)~\citep{dsa} during the distillation process. 
To evaluate the distilled datasets, we train ConvNet classifiers from scratch for 1{,}000 epochs using SGD. 
All reported results are averaged over five independent runs, and we provide both the mean and standard deviation.
Experiments on CORe50-Hard are conducted using a single A6000 GPU, while all other experiments use a single RTX 3090 GPU.

\noindent\textbf{Evaluation Protocols.}\quad
We evaluate the distilled datasets under two distinct domain-shift settings: multi-domain generalization (MDG) and single-domain generalization (SDG, Section~\ref{sec:ext_sdg}). 
For MDG, we adopt a strict leave-one-domain-out protocol: we select one domain as the strictly unseen target for evaluation and use all remaining domains as sources for distillation, reporting the top-1 classification accuracy (\%) on the target domain. 
For SDG, we distill the dataset from a single source domain and evaluate the trained model on each of the remaining domains as unseen targets. For every source choice, we report the average top-1 accuracy across all corresponding target domains.

\subsection{Performance Comparison}
We evaluate the OOD generalization performance of SGS across three benchmark datasets under two IPC settings of 10 and 20. 
To rigorously assess the effectiveness of our approach, we compare it against a diverse set of baselines: whole-dataset training~(\textit{Whole}) as an upper bound, coreset selection methods~(\textit{Random}, \textit{Herding}~\citep{herding}), and state-of-the-art DM methods~(\textit{DM}~\citep{dm}, \textit{DataDAM}~\citep{datadam}, \textit{M3D}~\citep{m3d}, \textit{HDD-DM}~\citep{hdd}). 
Note that, for a fair comparison, we disable the multi-formation function in \textit{M3D}.
To validate the plug-and-play compatibility of the proposed SGS, we also integrate it into both \textit{DM} and \textit{HDD-DM}. 
In the Tables, the best and second-best results are highlighted in \textbf{bold} and \underline{underlined}, respectively. 

\noindent\textbf{Multi-domain Generalization Performance.}\quad
We evaluate the proposed SGS primarily in the multi-domain generalization (MDG) setting using a leave-one-domain-out protocol. As reported in Table~\ref{tab:mdg}, we observe that recent state-of-the-art DD methods, such as \textit{DataDAM} and \textit{M3D}, often struggle in this multi-domain distillation setting with poor generalization performance. In several cases, they even fall behind the vanilla DM baseline. This observation aligns with the findings of \citet{dam} that conventional DD methods are fundamentally designed under the assumption of a single, homogeneous source domain. When distilling a multi-domain source dataset, their matching objectives indiscriminately entangle conflicting domain characteristics, leading to sub-optimal distillation. 
In contrast, the proposed SGS effectively disentangles these signals and consistently improves the average OOD accuracy of its underlying DM methods. Interestingly, the performance margins vary depending on the synthetic dataset budget. At a highly constrained budget (IPC=10), SGS yields substantial absolute gains (e.g., +2.6\% for DM on Digits-DG, +3.4\% for HDD-DM on CORe50-Hard), demonstrating its exceptional data efficiency. 
Furthermore, on datasets with severe visual domain shifts like PACS, SGS yields robust and consistent absolute gains across both IPC settings (+3.0\% and +2.2\% over DM and HDD-DM at IPC=10, and +2.1\% and +2.6\% at IPC=20). We attribute this to the fact that simply increasing the data budget is insufficient to overcome severe domain shifts without explicitly disentangling class-discriminative from domain-specific information, a capability uniquely provided by SGS. Overall, these results indicate that SGS provides stable OOD improvements across diverse dataset scales and resolutions, firmly validating its plug-and-play compatibility across different DM architectures.

\begin{table}[t]
    \caption{Single-domain generalization performance on the Digits-DG dataset with IPC=10 and IPC=20.}
    \centering
    \resizebox{0.97\linewidth}{!}{
    \begin{tabular}{c|ccccc|ccccc}
    \toprule
    IPC & \multicolumn{5}{c|}{10} & \multicolumn{5}{c}{20} \\
    Source & M & N & V & Y & Avg. & M & N & V & Y & Avg. \\ \midrule
    \multicolumn{1}{l|}{Whole}
    & 45.3\pmval0.7 & 62.3\pmval0.6 & 64.6\pmval0.8 & 70.4\pmval0.3 & 60.7\pmval0.6
    & 45.3\pmval0.7 & 62.3\pmval0.6 & 64.6\pmval0.8 & 70.4\pmval0.3 & 60.7\pmval0.6\\
    \midrule
    \multicolumn{1}{l|}{Random} 
    & 38.1\pmval1.0 & 30.1\pmval0.6 & 29.1\pmval1.0 & 33.2\pmval0.5 & 32.6\pmval0.8
    & 39.7\pmval0.7 & 49.7\pmval1.2 & 41.4\pmval2.1 & 41.6\pmval2.0 & 43.1\pmval1.5\\
    \multicolumn{1}{l|}{Herding} 
    & 38.9\pmval1.2 & 48.5\pmval0.7 & 39.7\pmval0.7 & 36.9\pmval0.8 & 41.0\pmval0.9
    & 40.4\pmval0.5 & 51.2\pmval1.1 & 45.4\pmval1.1 & 43.1\pmval1.2 & 45.0\pmval1.0\\
    \midrule
    \multicolumn{1}{l|}{DM} 
    & 42.6\pmval0.8 & \underline{54.7}\pmval0.6 & 45.0\pmval1.3 & \underline{54.8}\pmval0.9 & 49.3\pmval0.9
    & 42.1\pmval1.5 & 57.2\pmval0.8 & 50.0\pmval0.8 & 55.3\pmval0.7 & 51.2\pmval1.0\\
    \rowcolor[gray]{0.9}
    \multicolumn{1}{l|}{\quad+SGS} 
    & \textcolor{red}{\textbf{45.5}}\pmval0.8 & \textcolor{red}{\textbf{55.6}}\pmval0.1 & \textcolor{red}{\textbf{53.2}}\pmval1.0 & \textcolor{red}{\textbf{55.4}}\pmval0.7 & \textcolor{red}{\textbf{52.4}}\pmval0.7
    & \textcolor{red}{\textbf{46.0}}\pmval1.5 & \textcolor{red}{\textbf{58.3}}\pmval0.6 & \textcolor{red}{\textbf{54.4}}\pmval0.3 & \textcolor{red}{\underline{55.8}}\pmval0.7 & \textcolor{red}{\textbf{53.6}}\pmval0.8\\
    \multicolumn{1}{l|}{HDD-DM} 
    & 42.7\pmval0.5 & \underline{54.7}\pmval0.5 & 46.4\pmval1.6 & 50.4\pmval1.0 & 48.6\pmval0.9
    & 42.5\pmval0.8 & 56.6\pmval0.2 & 49.4\pmval0.6 & 53.3\pmval0.4 & 50.5\pmval0.5\\
    \rowcolor[gray]{0.9}
    \multicolumn{1}{l|}{\quad+SGS} 
    & \textcolor{red}{\underline{43.4}}\pmval0.8 & \textcolor{red}{\textbf{55.6}}\pmval0.4 & \textcolor{red}{\underline{50.4}}\pmval0.8 & \textcolor{red}{53.9}\pmval0.8 & \textcolor{red}{\underline{50.8}}\pmval0.7
    & \textcolor{red}{\underline{44.3}}\pmval0.8 & \textcolor{red}{\underline{58.0}}\pmval0.8 & \textcolor{red}{\underline{53.2}}\pmval1.2 & \textcolor{red}{\textbf{56.7}}\pmval0.6 & \textcolor{red}{\underline{53.1}}\pmval0.9\\
    \bottomrule
    \end{tabular}
    }
    \label{tab:sdg}
\end{table}

\noindent\textbf{Single-domain Generalization Performance.}\quad
We further evaluate SGS in the single-domain generalization (SDG) setting, where only one source domain is available for distillation. 
To broaden the applicability of SGS, we partition the source dataset into $K$ pseudo-domains. 
Directly motivated by the empirical observation in Figure~\ref{fig:umap_domain} that style statistics effectively cluster by domain, we extract these statistics from intermediate feature maps and apply K-means clustering to identify latent domain-specific patterns (detailed procedures are provided in Appendix~\ref{sec:ext_sdg}). 
These pseudo-domains then serve as a proxy for ground-truth domain splits, enabling the SGS. 
Table~\ref{tab:sdg} reports the average top-1 accuracy across the three unseen target domains for each choice of source domain, demonstrating consistent OOD improvements across all source domains for both DM baselines. 
These results indicate that the pseudo-domains identified through feature-statistic clustering, while inherently noisier than ground-truth domain labels, provide a highly effective proxy for the cross-domain agreement signal, allowing SGS to operate robustly even in the constrained single-source setting.
Additional results, including visualizations and analysis of $K$, are provided in Appendix~\ref{sec:more_sdg}.

\begin{table*}[t]
\begin{minipage}[t]{0.46\linewidth}
    \caption{Performance comparison with post-hoc DG methods on PACS at IPC=10.}
    \centering
    \begin{threeparttable}
    \resizebox{\linewidth}{!}{%
    \setlength{\tabcolsep}{1.8pt}
    \begin{tabular}{c|cccc|c}
    \toprule
    Target & P & A & C & S & Overhead \\ \midrule
    \multicolumn{1}{l|}{DM}
    & \underline{48.0}\pmval0.7 & 27.9\pmval0.5 & 33.8\pmval1.5 & 32.6\pmval2.2 & 1.0$\times$ \\ 
    \multicolumn{1}{l|}{\quad+FACT$^\dagger$}
    & 45.3\pmval1.4 & \underline{31.6}\pmval0.4 & \textbf{35.5}\pmval1.5 & 34.3\pmval2.0 & 2.1$\times$ \\ 
    \multicolumn{1}{l|}{\quad+MixStyle$^\dagger$}
    & 47.1\pmval1.0 & 27.8\pmval0.5 & 33.9\pmval1.0 & 31.8\pmval2.6 & 1.2$\times$ \\ 
    \multicolumn{1}{l|}{\quad+AdvFreq$^\dagger$}
    & 41.4\pmval2.2 & \textbf{31.7}\pmval0.5 & 31.5\pmval0.5 & \textbf{40.5}\pmval0.6 & 6.8$\times$ \\ 
    \rowcolor[gray]{0.9}
    \multicolumn{1}{l|}{\quad+SGS}
    & \textbf{53.8}\pmval0.4 & 28.4\pmval1.1 & \underline{35.1}\pmval1.5 & \underline{37.2}\pmval1.8 & \textbf{1.0$\times$}\\
    \bottomrule
    \end{tabular}%
    }
    \begin{tablenotes}
        \scriptsize
        {\item[$\dagger$]\!\!\!\! Post-hoc methods applied on top of the DM-distilled dataset.}
    \end{tablenotes}
    \end{threeparttable}
    \label{tab:dg}
\end{minipage}
\hfill
\begin{minipage}[t]{0.51\linewidth}
    \caption{Effect of $\boldsymbol{g}^\text{class}$, $\boldsymbol{g}^\text{domain}$, and $\boldsymbol{g}$, evaluated on Digits-DG at IPC=10.}
    \centering
    \resizebox{0.96\linewidth}{!}{
    \setlength{\tabcolsep}{2pt}
    \begin{tabular}{ccc|ccccc}
    \toprule
    $\boldsymbol{g}^\text{class}$ & $\boldsymbol{g}^\text{domain}$ & $\boldsymbol{g}$ & M & N & V & Y & Avg. \\
    \midrule
     &  & \checkmark &
    85.3\pmval0.7 & 51.4\pmval1.5 & 53.6\pmval0.9 & 69.9\pmval0.9 & 65.1\pmval1.0 \\
    
    \checkmark & & & 
    85.5\pmval0.9 &52.6\pmval1.2 &52.4\pmval1.3 &68.9\pmval1.5 &64.9\pmval1.2 \\
   
    & \checkmark & & 
    65.8\pmval2.5 &42.5\pmval1.0 &27.2\pmval0.7 &43.1\pmval1.2 &44.7\pmval1.4 \\
    
    \checkmark & \checkmark & & 
    88.6\pmval0.9 &51.8\pmval1.0 &52.7\pmval0.8 &67.1\pmval0.5 &65.1\pmval0.8 \\
    
    \rowcolor[gray]{0.9}        
    \checkmark & \checkmark & \checkmark & 
    \textbf{90.3}\pmval0.4 &\textbf{54.1}\pmval0.9 &\textbf{55.3}\pmval0.4 &\textbf{71.0}\pmval0.8 &\textbf{67.7}\pmval0.6 \\
    \bottomrule
    \end{tabular}}
    \label{tab:ablation}
\end{minipage}
\vspace{-4mm}
\end{table*}

\noindent\textbf{Comparison with Post-hoc Domain Generalization.}\quad
We also compare the performance by applying existing DG methods (\textit{FACT}~\citep{fact}, \textit{MixStyle}~\citep{mixstyle}, \textit{AdvFreq}~\citep{advfreq}) post-hoc during downstream training. However, as shown in Table~\ref{tab:dg}, this approach struggles to guarantee consistent improvements. Because conventional DG methods assume the high diversity of large-scale real data, applying them to highly compressed distilled datasets often yields marginal gains or even degrades performance (e.g., \textit{MixStyle}). Furthermore, complex online augmentations in methods like \textit{FACT} and \textit{AdvFreq} drastically increase downstream training time by 2.1$\times$ and 6.8$\times$, fundamentally contradicting the efficiency goal of dataset distillation. In contrast, SGS inherently embeds domain-invariant signals during the offline distillation process. Thus, it achieves the best average OOD generalization without incurring any additional computational overhead (1.0$\times$) during downstream training.

\subsection{Ablation Study}
\begin{wrapfigure}[12]{r}{0.63\linewidth} 
  \vspace{-5mm}
    \centering
    \subfloat[]{\includegraphics[width=0.315\linewidth]{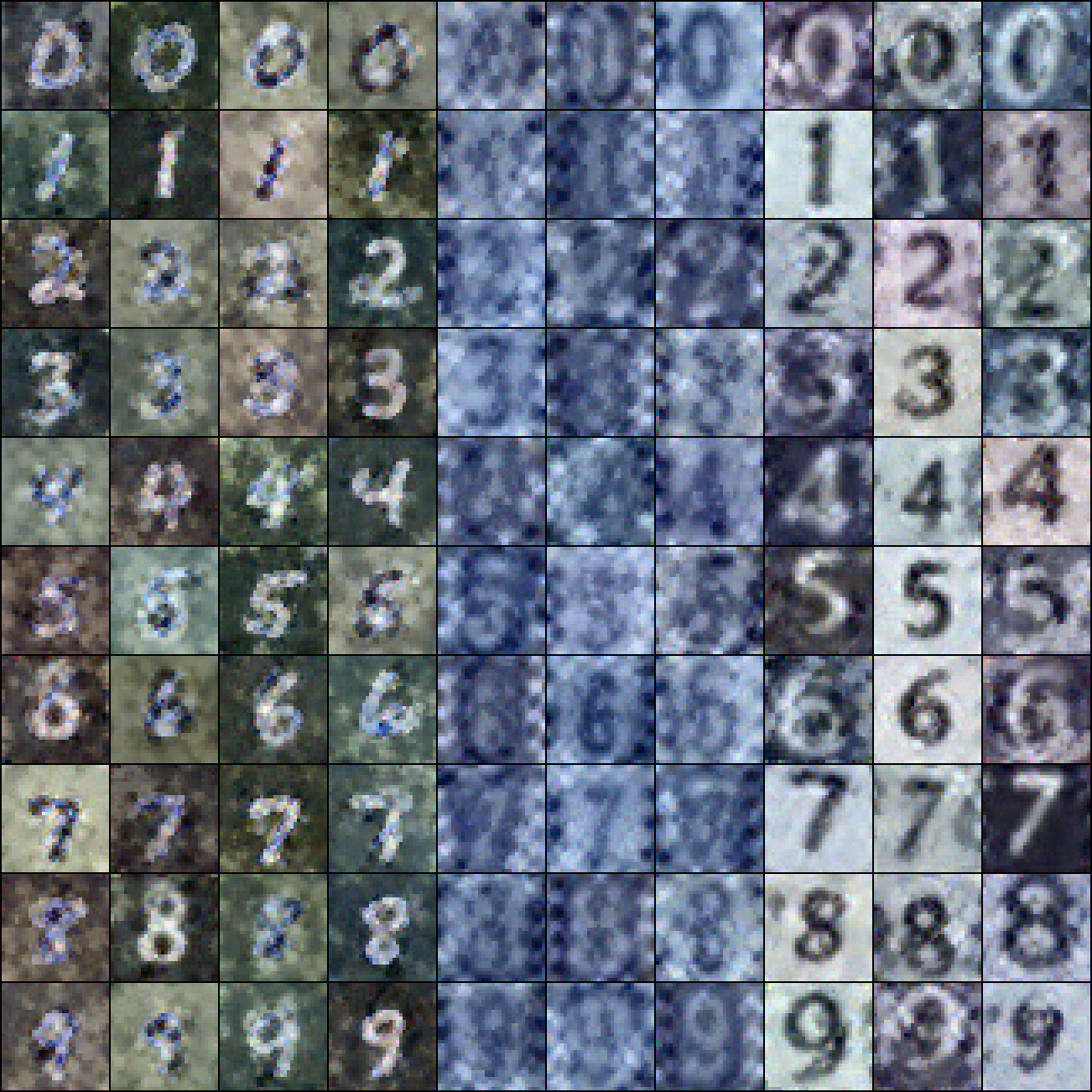}\label{fig:full_sgs}}
    \hfill
    \subfloat[]{\includegraphics[width=0.315\linewidth]{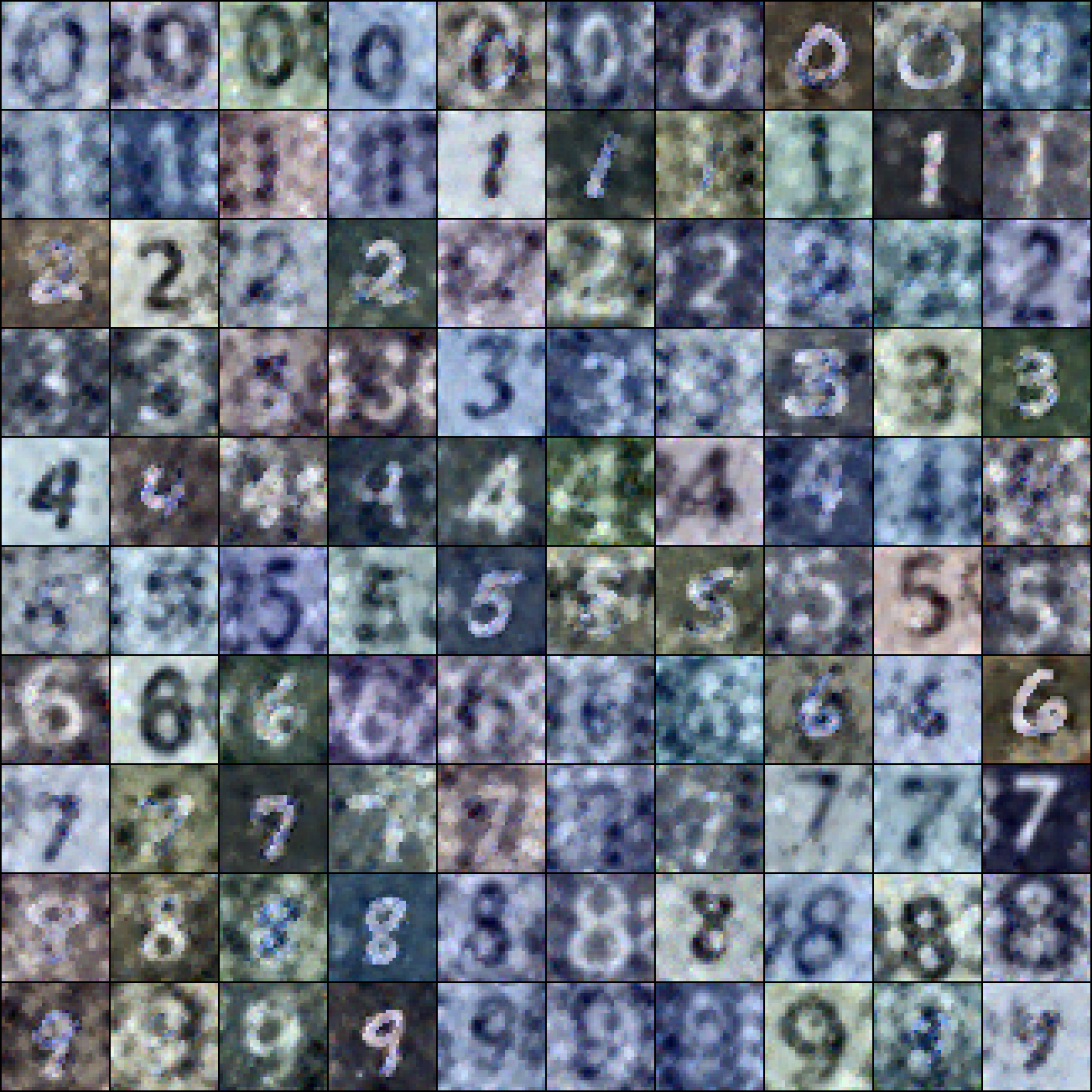}\label{fig:class_only}}
    \hfill
    \subfloat[]{\includegraphics[width=0.315\linewidth]{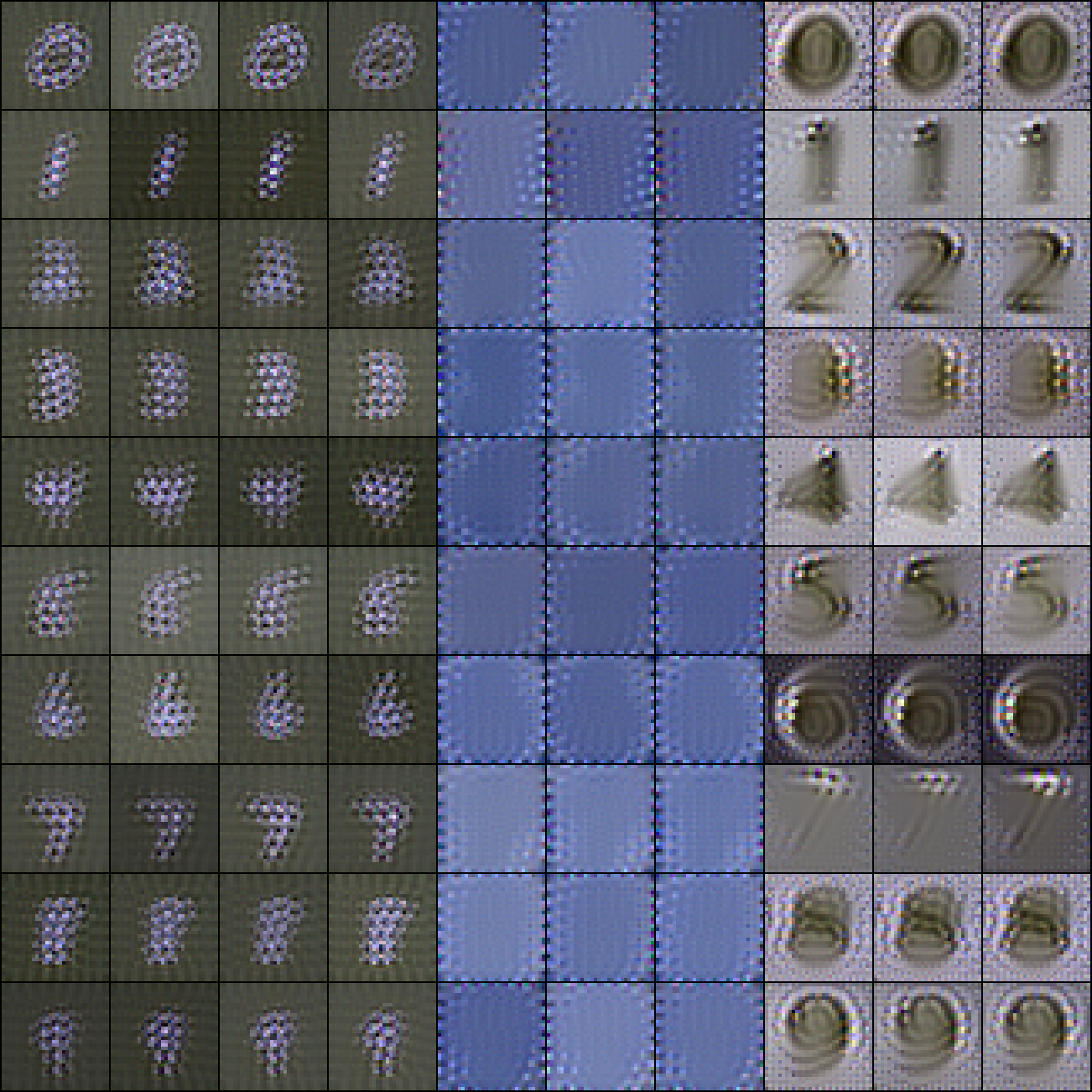}\label{fig:domain_only}}
    \vspace{-2mm}
    \caption{Visualization of distilled Digits-DG images (IPC=10) under different gradient update strategies: (a)~SGS, (b)~$\boldsymbol{g}^\text{class}$ only, and (c)~$\boldsymbol{g}^\text{domain}$ only.}
    \label{fig:distilled}
\end{wrapfigure}
To analyze the contribution of each gradient component in Eq.~\ref{eq:sgs_final}, we conduct an ablation study on Digits-DG (IPC=10). 
The results are reported in Table~\ref{tab:ablation}, where each column represents the OOD accuracy for a specific unseen target domain.
When using $\boldsymbol{g}$ alone, SGS reduces to standard DM and serves as the baseline. 
Replacing $\boldsymbol{g}$ with $\boldsymbol{g}^\text{class}$ alone yields comparable OOD accuracy on M and N but slightly degrades performance on V and Y, suggesting that the cross-domain agreement signal alone is insufficient as a primary update direction. 
This is also reflected in Figure~\ref{fig:class_only}, where the distilled images exhibit clear class semantics but lack visual diversity, sharing highly uniform domain characteristics across samples. 
In contrast, relying solely on $\boldsymbol{g}^\text{domain}$ causes severe performance degradation across all target domains. 
As shown in Figure~\ref{fig:domain_only}, while these images capture domain-specific appearances—which inherently include some recognizable digit-like structures—they fail to encode fine-grained class-discriminative features required for accurate classification. 
Combining only the two SGS signals ($\boldsymbol{g}^\text{class} + \boldsymbol{g}^\text{domain}$, without $\boldsymbol{g}$) improves OOD accuracy on M and N, but falls short of the vanilla DM baseline on V and Y, indicating that the standard DM gradient $\boldsymbol{g}$ is indispensable as the primary optimization direction. 
Finally, combining all three components yields the best performance across every target domain. This demonstrates that $\boldsymbol{g}^\text{class}$ and $\boldsymbol{g}^\text{domain}$ provide essential and complementary refinements to the primary signal $\boldsymbol{g}$. As visualized in Figure~\ref{fig:full_sgs}, the resulting distilled images successfully preserve faithful class-discriminative information while exhibiting rich cross-domain visual diversity.

\subsection{Further Analysis}

\noindent\textbf{Visualization of $\boldsymbol{g}$ and $\boldsymbol{g}^\text{class}$.}\quad
To verify whether the proposed spectral decomposition effectively isolates class-discriminative information, we visualize the standard DM gradient $\boldsymbol{g}$ and the class-discriminative signal $\boldsymbol{g}^{\text{class}}$ extracted by SGS in Figure~\ref{fig:grad_class}. 
While $\boldsymbol{g}^{\text{class}}$ exhibits sharper, highly recognizable digit structures, $\boldsymbol{g}$ displays visually noisier patterns, where the underlying digit shapes are partially obscured by textural variations caused by domain-specific signals. 
This qualitative contrast strongly supports our core motivation discussed in Section~\ref{sec:dissect}: while the standard DM gradient inevitably entangles class-discriminative and domain-specific information, $\boldsymbol{g}^{\text{class}}$—derived from cross-domain agreement in the spectral domain—successfully distills and recovers the shared class signal more clearly.

\begin{wrapfigure}{r}{0.5\linewidth} 
    \vspace{-4mm}
    \centering
    \captionof{table}{In-distribution performance on Digits-DG (IPC=10) under the MDG setting.}
    \label{tab:id_mdg}
    \resizebox{\linewidth}{!}{
    \setlength{\tabcolsep}{4pt}
    \begin{tabular}{c|ccccc}
        \toprule
        Target & M & N & V & Y & Avg. \\
        \midrule
        \multicolumn{1}{l|}{DM} & 73.6\pmval{0.5} & 82.8\pmval{0.6} & 82.3\pmval{0.4} & 77.7\pmval{0.4} & 79.1\pmval{0.5} \\
        \rowcolor[gray]{0.9}
        \multicolumn{1}{l|}{\quad+ SGS} & 73.7\pmval{0.9} & 82.3\pmval{0.8} & 81.7\pmval{0.4} & 77.9\pmval{0.5} & 78.9\pmval{0.7} \\
        \midrule
        \multicolumn{1}{l|}{HDD-DM} & 74.1\pmval{0.6} & 83.2\pmval{0.4} & 82.1\pmval{0.3} & 77.5\pmval{0.2} & 79.2\pmval{0.4} \\
        \rowcolor[gray]{0.9}        
        \multicolumn{1}{l|}{\quad+ SGS} & 74.4\pmval{0.4} & 82.4\pmval{0.2} & 83.4\pmval{0.2} & 77.5\pmval{0.6} & 79.4\pmval{0.4} \\
        \bottomrule
    \end{tabular}%
    }
\end{wrapfigure}
\noindent\textbf{In-distribution Performance.}\quad
Table~\ref{tab:id_mdg} reports the ID performance on Digits-DG at IPC=10, where the test domains coincide with the source domains used for distillation. We observe that SGS yields ID performance comparable to that of the underlying base methods on average, despite substantially improving OOD generalization performance.

\begin{figure}[t]
\centering
\begin{minipage}[c]{0.49\linewidth}
    \centering
    \subfloat[$\boldsymbol{g}$]{\includegraphics[width=0.488\linewidth]{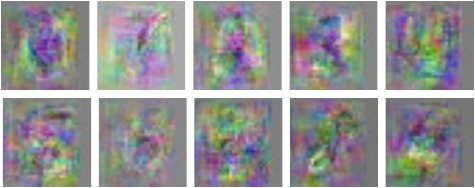}}\hfill
    \subfloat[$\boldsymbol{g}^\text{class}$]{\includegraphics[width=0.488\linewidth]{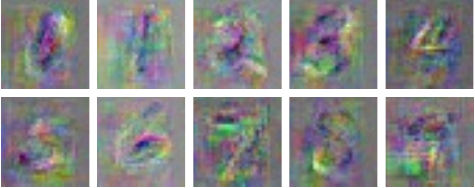}}
    \caption{Visualization of $\boldsymbol{g}$ and $\boldsymbol{g}^{\text{class}}$ on Digits-DG~(Upper row: 0 to 4, Lower row: 5 to 9).}
    \label{fig:grad_class}
\end{minipage}\hfill
\begin{minipage}[c]{0.49\linewidth}
    \centering
    \subfloat[$\lambda_c$]{\includegraphics[width=0.49\linewidth]{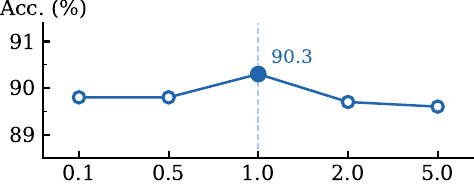}}\hfill
    \subfloat[$\lambda_d$]{\includegraphics[width=0.49\linewidth]{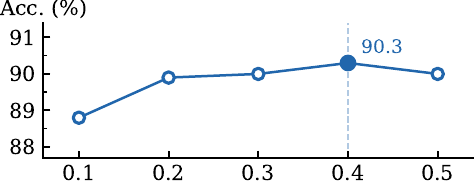}}
    \caption{Performance changes according to (a)~$\lambda_c$ and (b)~$\lambda_d$ on Digits-DG.}
    \label{fig:hyp}
\end{minipage}%
\end{figure}

\noindent\textbf{Hyperparameter Sensitivity.}\quad
Figure~\ref{fig:hyp} illustrates the OOD performance changes on Digits-DG (IPC=10, Target M) with respect to hyperparameters $\lambda_c$ and $\lambda_d$. 
SGS consistently outperforms the baseline DM across the entire range of parameters, affirming its robustness and practical plug-and-play capability. 
The performance trends reflect the inherent trade-offs between the gradient components. As $\lambda_c$ increases, the update becomes more heavily reliant on the cross-domain consensus, which tends to reduce the visual diversity of synthetic samples, as qualitatively observed in Figure~\ref{fig:class_only}.
On the other hand, a larger $\lambda_d$ shifts the focus toward preserving domain-specific characteristics, which can eventually diminish the fine-grained, class-discriminative information necessary for accurate classification. 
These results confirm that a balanced integration of both signals is essential to achieve optimal OOD generalization.

\section{Conclusion}
In this paper, we studied the problem of dataset distillation under domain shift, which explicitly targets robustness to unseen target domains rather than fidelity to the source distribution alone. To address this, we proposed Spectral Gradient Surgery (SGS), which augments the standard DM update with two additional gradient signals derived from decomposition in the spectral domain: a class-discriminative signal that reinforces components agreed across source domains, and a domain-specific signal that encourages diversity of distilled dataset. 
The proposed SGS is plug-and-play compatible with existing DM methods, requiring no modification to their core objectives. 
Experimental results on the Digits-DG, PACS, and CORe50-Hard datasets demonstrate that SGS consistently improves OOD generalization performance across diverse-scale benchmarks, providing the first effective approach to domain generalizable dataset distillation.

\noindent\textbf{Limitation and Future Work.}\quad
SGS is designed and validated within the distribution matching framework, which we focus on for its strong performance and computational efficiency among DD families. While the underlying principle—decomposing gradient signals via cross-domain agreement in the spectral domain—is conceptually applicable to other paradigms such as gradient matching and trajectory matching, extending SGS to these settings is left for future work.

\bibliographystyle{plainnat}
\bibliography{references}

\newpage
\appendix

\section*{Appendix}

\section{Specification of Datasets}\label{app:dataset}
Table~\ref{tab:dataset_stats} summarizes the detailed statistics of the three benchmark datasets evaluated in our experiments. 
Furthermore, to illustrate the visual domain shifts present in these datasets, we provide representative examples in Figure~\ref{fig:dataset_examples}, where rows and columns correspond to distinct domains and classes, respectively.

\begin{table}[h!]
\centering
\caption{Detailed statistics of the benchmark datasets evaluated in our experiments.}
\label{tab:dataset_stats}
\resizebox{\textwidth}{!}{
\setlength{\tabcolsep}{1pt}
\begin{tabular}{lccc}
\toprule
 & Digits-DG~\citep{digitsDG} & PACS~\citep{pacs} & CORe50-Hard~\citep{core50} \\
\midrule
\# Domains & 4 & 4 & 4 \\
\# Classes & 10 & 7 & 10 \\
\midrule
\multirow{4}{*}{Domains}
& MNIST~\citep{mnist} & Photo & s5 \\
& MNIST-M~\citep{syn} & Art painting & s7 \\
& SVHN~\citep{svhn} & Cartoon & s9 \\
& SYN~\citep{syn} & Sketch & s10 \\
\midrule
\# Images per domain 
& 6,000 / 6,000 / 6,000 / 6,000 
& 1,670 / 2,048 / 2,344 / 3,929 
& 14,966 / 14,994 / 14,994 / 14,986 \\
Resolution 
& $32{\times}32$ 
& $64{\times}64$ 
& $128{\times}128$ \\
\bottomrule
\end{tabular}
}
\end{table}

\begin{figure}[h!]
    \centering
        \includegraphics[width=0.85\linewidth]{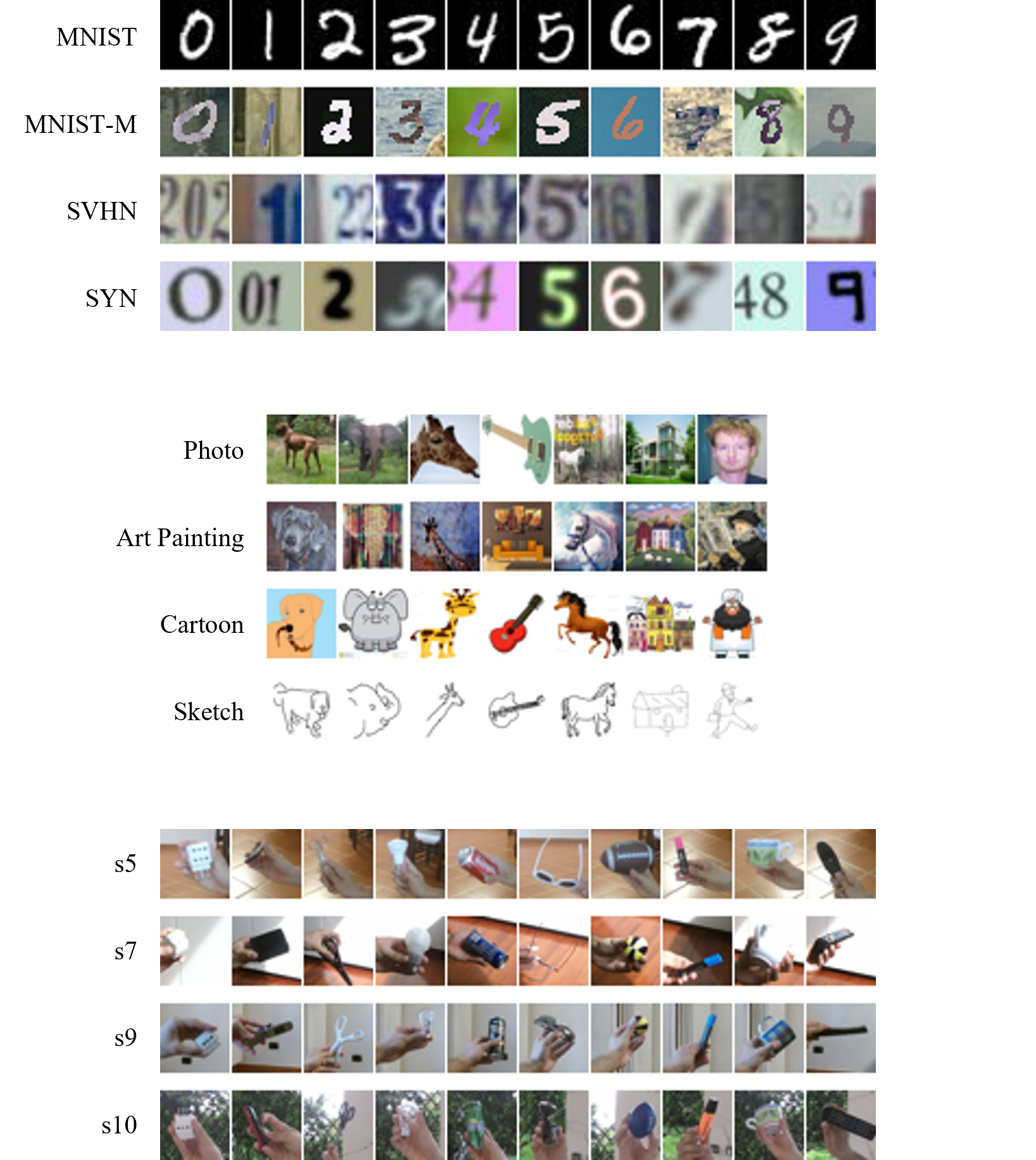}
    \caption{Visual examples from the three benchmark datasets. For each dataset, rows correspond to distinct domains, while columns represent different classes. Note the severe visual domain shifts across the source domains.}
    \label{fig:dataset_examples}
\end{figure}

\section{Theoretical Justification of SGS}\label{app:theory}

We provide a theoretical analysis demonstrating how the spectral gradient decomposition in SGS mathematically recovers the domain-invariant class signal while effectively attenuating domain-specific noise. Our analysis builds on standard concepts from directional statistics.

\subsection{Setup and Assumptions}

For a synthetic sample $\hat{x}_i$ and a specific frequency component $k$, let $G_{i,k}^s$ denote the per-domain spectral gradient of source domain $s \in \{1,\dots,S\}$. 
We posit that this gradient is composed of two components:
\begin{equation}
G_{i,k}^s \;=\; G_{i,k}^{\mathrm{shared}} \;+\; N_{i,k}^s,
\label{eq:decomp}
\end{equation}
where $G_{i,k}^{\mathrm{shared}}$ is a domain-invariant component encoding the class-discriminative signal, and $N_{i,k}^s$ is the domain-specific noise. 
By expressing the gradient in polar form as $G_{i,k}^s = A_{i,k}^s \, e^{j\phi_{i,k}^s}$, we make the following mild assumptions:

\begin{assumption}[Phase Distribution]
\label{assump:phase}
For a given frequency $k$, the phase angles $\{\phi_{i,k}^s\}_{s=1}^{S}$ are drawn from a circular distribution characterized by a \textbf{true phase concentration} $\rho_k \in [0,1]$. Here, $\rho_k \to 1$ indicates that the phases are highly concentrated around a single direction (strong cross-domain agreement), whereas $\rho_k \to 0$ indicates a uniform, random spread of phases (no agreement).
\end{assumption}

\begin{assumption}[Bounded Magnitude]
\label{assump:mag}
The magnitudes $\{A_{i,k}^s\}_{s=1}^{S}$ are strictly bounded within $[m_k, M_k]$ where $0 < m_k \le M_k < \infty$, and are independent of the phases.
\end{assumption}

\subsection{Main Result}

\begin{proposition}[Disentanglement of Class-Discriminative Signals]\label{prop:recovery}
Under Assumptions~\ref{assump:phase} and \ref{assump:mag}, as the number of source domains $S$ increases, the SGS class-discriminative signal $G_{i,k}^{\mathrm{class}} = \bar{G}_{i,k} \cdot r_{i,k}$ defined in Eq.~\ref{eq:sgs_decompose} behaves as follows:
\begin{enumerate}
    \item[(i)] (\textbf{Preservation of domain-invariant signals}) If $\rho_k \to 1$, the SGS update perfectly recovers the true signal: $G_{i,k}^{\mathrm{class}} \to G_{i,k}^{\mathrm{shared}}$.
    \item[(ii)] (\textbf{Attenuation of domain-specific signals}) If $\rho_k \to 0$, the SGS update inherently suppresses the noise: $\mathbb{E}[|G_{i,k}^{\mathrm{class}}|] = \mathcal{O}(1/S)$.
\end{enumerate}
\end{proposition}

\begin{proof} \textbf{Step 1: Convergence of the Resultant Length $r_{i,k}$.}\quad
To analyze the asymptotic behavior of $r_{i,k}$, we divide both the numerator and the denominator by $S$. By the Law of Large Numbers, the empirical averages converge to their respective expectations as $S \to \infty$:
\begin{equation}
    r_{i,k} = \frac{\left| \frac{1}{S} \sum_{s=1}^S A_{i,k}^s \, e^{j\phi_{i,k}^s} \right|}{\frac{1}{S} \sum_{s=1}^S A_{i,k}^s + \frac{\epsilon}{S}} \xrightarrow{S \to \infty} \frac{\left| \mathbb{E}[A_{i,k}^s \, e^{j\phi_{i,k}^s}] \right|}{\mathbb{E}[A_{i,k}^s]}.
\end{equation}
Since the magnitude $A_{i,k}^s$ and the phase $\phi_{i,k}^s$ are independent, the expectation in the numerator separates as $\mathbb{E}[A_{i,k}^s] \cdot \mathbb{E}[e^{j\phi_{i,k}^s}]$. Therefore, the term $\mathbb{E}[A_{i,k}^s]$ cancels out:
\begin{equation}
    r_{i,k} \xrightarrow{S \to \infty} \frac{\mathbb{E}[A_{i,k}^s] \cdot |\mathbb{E}[e^{j\phi_{i,k}^s}]|}{\mathbb{E}[A_{i,k}^s]} = \rho_k,
\end{equation}
which guarantees that the empirical resultant length robustly converges to the true phase concentration.

\textbf{Step 2: Preservation of Domain-Invariant Signals ($\rho_k \to 1$).} \quad
In directional statistics, the true phase concentration determines the circular variance as $V(\phi_{i,k}) = 1 - \rho_k$. As $\rho_k \to 1$, the variance $V(\phi_{i,k}) \to 0$, implying that the phases across all domains become perfectly aligned to a single consensus direction. Consequently, the domain-specific noise strictly vanishes ($N_{i,k}^s \to 0$), and all per-domain gradients converge to the shared signal:
\begin{equation}
    G_{i,k}^s \xrightarrow{\rho_k \to 1} G_{i,k}^{\mathrm{shared}} \quad \text{for all } s \in \{1, \dots, S\}.
\end{equation}
Under this condition, the empirical average gradient perfectly isolates the shared signal:
\begin{equation}
    \bar{G}_{i,k} = \frac{1}{S}\sum_{s=1}^S G_{i,k}^s = G_{i,k}^{\mathrm{shared}}.
\end{equation}
Furthermore, because all vectors point in the identical direction, the magnitude of their sum precisely equals the sum of their individual magnitudes. Thus, the resultant length reaches its theoretical maximum:
\begin{equation}
    r_{i,k} = \frac{\left| \sum_{s=1}^S G_{i,k}^s \right|}{\sum_{s=1}^S |G_{i,k}^s|} = \frac{\left| S \cdot G_{i,k}^{\mathrm{shared}} \right|}{S \cdot \left| G_{i,k}^{\mathrm{shared}} \right|} = 1.
\end{equation}
Multiplying these two terms yields $G_{i,k}^{\mathrm{class}} = \bar{G}_{i,k} \cdot r_{i,k} = G_{i,k}^{\mathrm{shared}} \cdot 1 = G_{i,k}^{\mathrm{shared}}$, mathematically proving that the true class-discriminative signal is reconstructed without any loss.

\textbf{Step 3: Attenuation of Domain-Specific Signals ($\rho_k \to 0$).} \quad
Conversely, as $\rho_k \to 0$, the phases are uniformly distributed ($\mathbb{E}[e^{j\phi_{i,k}^s}] = 0$), making the spectral gradients act as independent zero-mean random variables. Summing $S$ such variables resembles a 2D random walk, where the expected magnitude of the sum $\sum_{s=1}^S G_{i,k}^s$ grows only by $\mathcal{O}(\sqrt{S})$. 
Consequently, the magnitude of the average gradient naturally shrinks as $\mathbb{E}[|\bar{G}_{i,k}|] = \mathcal{O}(1/\sqrt{S})$. 
Similarly, since the denominator of $r_{i,k}$ grows linearly by $\mathcal{O}(S)$, the resultant length also shrinks at the same rate: $\mathbb{E}[r_{i,k}] = \mathcal{O}(1/\sqrt{S})$.
Finally, multiplying these two terms yields $\mathbb{E}[|G_{i,k}^{\mathrm{class}}|] = \mathcal{O}(1/S)$, indicating fast suppression of domain-specific signals.
\end{proof}

\subsection{Interpretation}

\begin{remark}[Variance-Regularized Gradient Matching]
\label{rem:varreg}
In directional statistics, the true phase concentration is strictly tied to the circular variance: $\rho_k = 1 - V(\phi_{i,k})$. Proposition~\ref{prop:recovery} mathematically demonstrates that SGS inherently acts as a variance-regularized filter. Frequency components exhibiting small phase variance (i.e., class-discriminative signals with high cross-domain agreement) are perfectly preserved, whereas components with high phase variance (i.e., domain-specific noise) are strongly and quadratically penalized. This ensures that the distilled dataset is optimized strictly on invariant semantics, aligning with the core objective of out-of-distribution generalization.
\end{remark}



\section{Extension to Single Domain Generalization}
\subsection{Implementation Details}~\label{sec:ext_sdg}
In practice, $\mathcal{D}$ may not provide explicit domain splits, which corresponds to the single domain generalization (SDG) setting. 
To extend SGS to this setting, we estimate pseudo-domain labels $\{1, \dots, K\}$ to substitute for ground-truth domain labels.
Specifically, motivated by the observation in Figure~\ref{fig:umap_domain} that feature statistics encode domain-specific information, we extract channel-wise mean and standard deviation of intermediate feature maps from $\psi$ for each sample, and apply K-means clustering~\citep{kmeans} over these style statistics to partition $\mathcal{D}$ into $K$ pseudo-domains, which are then used in place of ground-truth domain splits throughout SGS.

\subsection{More Results in SDG Setting}\label{sec:more_sdg}

\noindent\textbf{Sensitivity of $K$.}\quad 
We analyze the sensitivity of the hyperparameter $K$, which determines the number of pseudo-domains, on the Digits-DG (IPC=10, target: M). 

\begin{wrapfigure}[7]{r}{0.26\linewidth} 
  \vspace{-5mm}
    \centering
    \includegraphics[width=\linewidth]{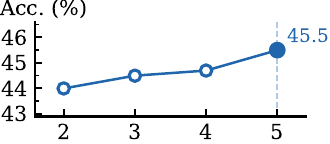}
    \vspace{-3mm}
    \caption{Performance changes according to $K$.}
    \vspace{-8mm}
    \label{fig:K}
\end{wrapfigure}
Figure~\ref{fig:K} illustrates the average OOD accuracy across different values of $K$. 
We observe that the OOD performance generally improves as $K$ increases. This empirical trend perfectly aligns with our theoretical finding in Proposition~\ref{prop:recovery}, which states stronger suppression of domain-specific signals with a larger number of source domains. 
However, a larger $K$ linearly scales up the distillation time overhead, as gradients must be computed independently for each pseudo-domain. 
Furthermore, an excessively large $K$ risks isolating outliers into overly sparse clusters, which injects noisy gradient signals and ultimately degrades the distillation performance.

\noindent\textbf{Visualization.}\quad 
Figure~\ref{fig:sdg_digit} visualizes the distilled images under the Single Domain Generalization (SDG) setting for each source domain, using an IPC of 10 and $K=3$ pseudo-domains. 
For the M source domain, the generated images reveal that each pseudo-domain captures distinct intra-domain variations, such as different writing styles, and digit angles. 
For the other source domains, the pseudo-domains exhibit clear visual distinctions, successfully grouping diverse color palettes and background textures. 
These visualizations confirm that our pseudo-domain clustering strategy faithfully uncovers meaningful latent sub-populations within a single source domain. 

\begin{figure}[h!]
    \vspace{-2mm}
    \centering
    \subfloat[M]{\includegraphics[width=0.24\linewidth]{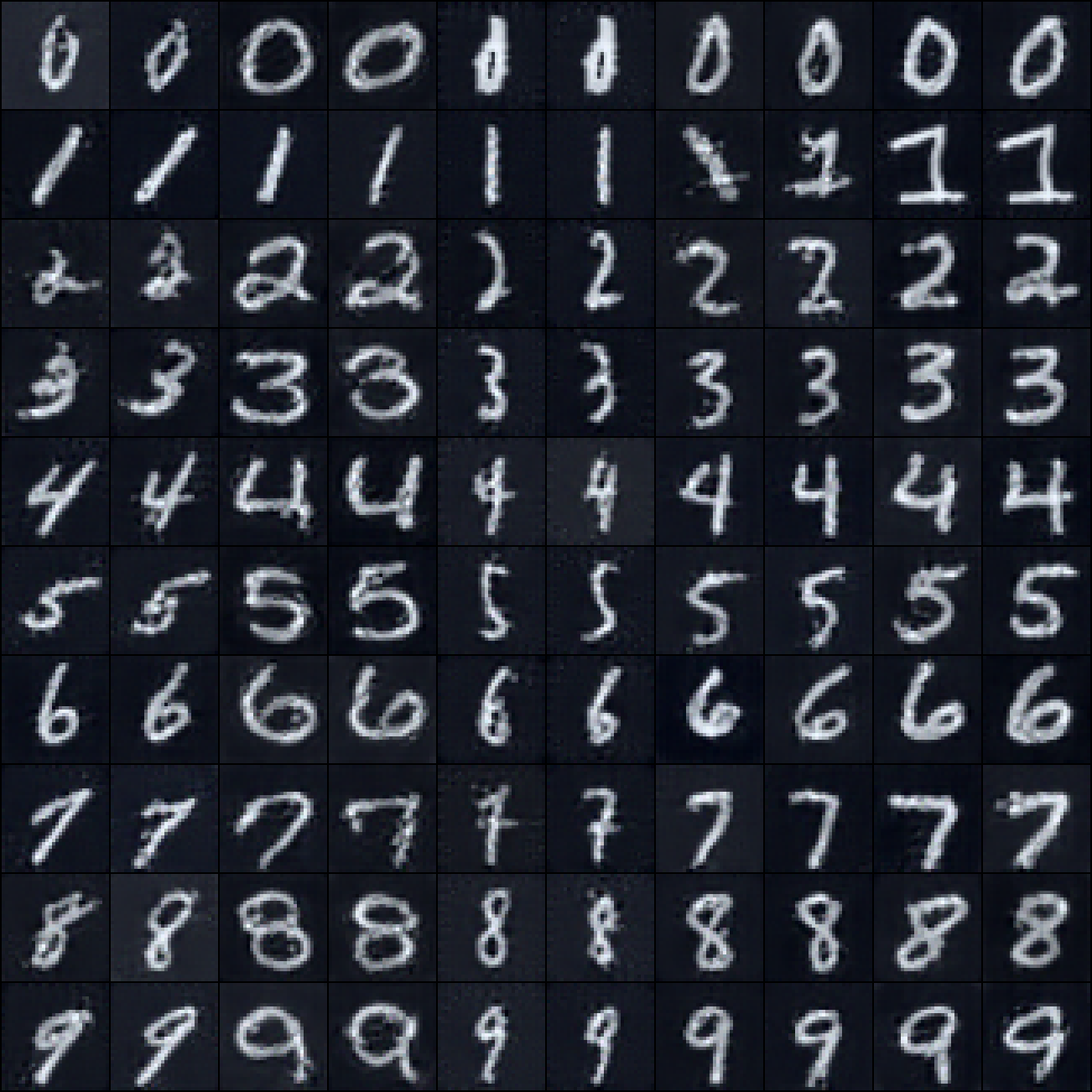}}\hfill
    \subfloat[N]{\includegraphics[width=0.24\linewidth]{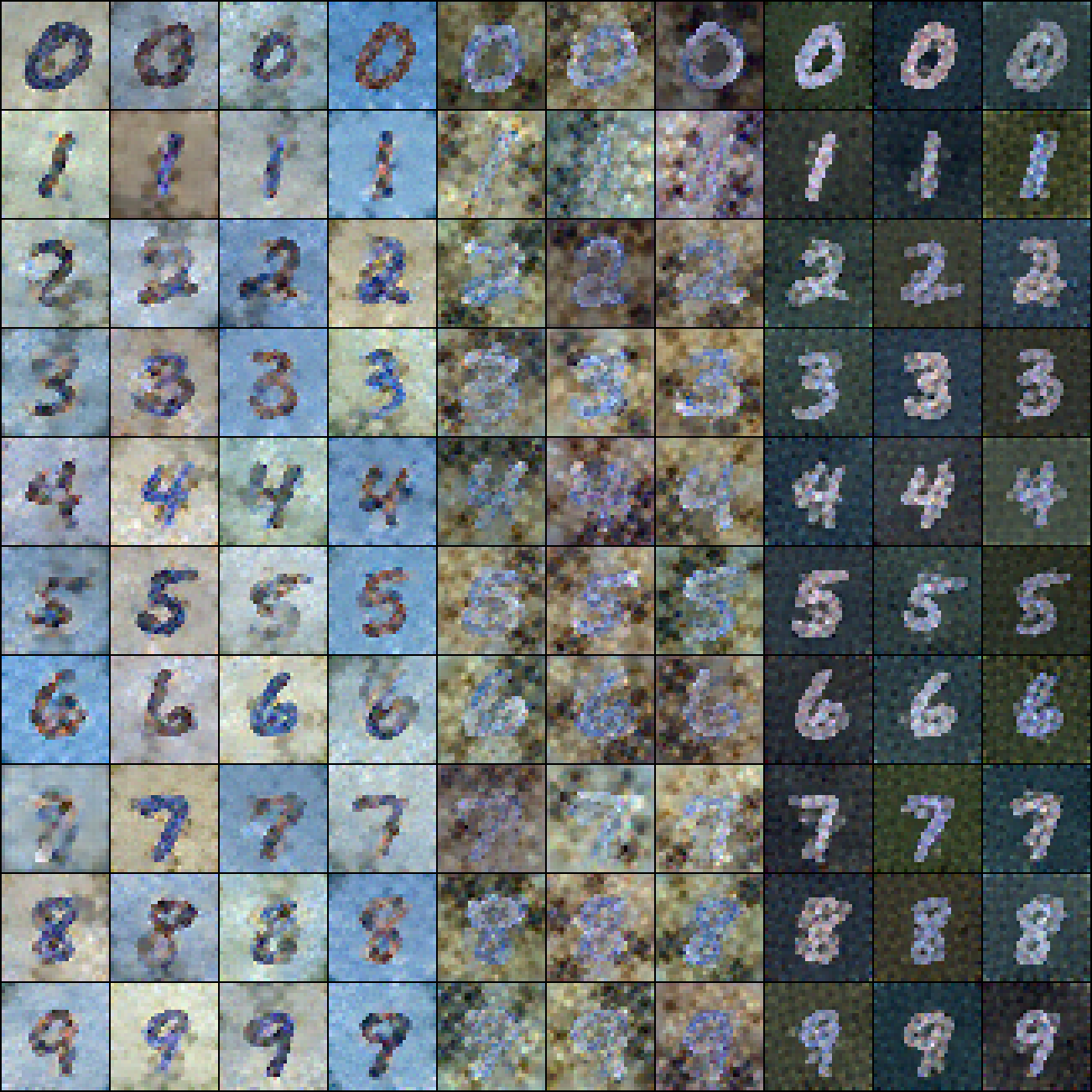}}\hfill
    \subfloat[V]{\includegraphics[width=0.24\linewidth]{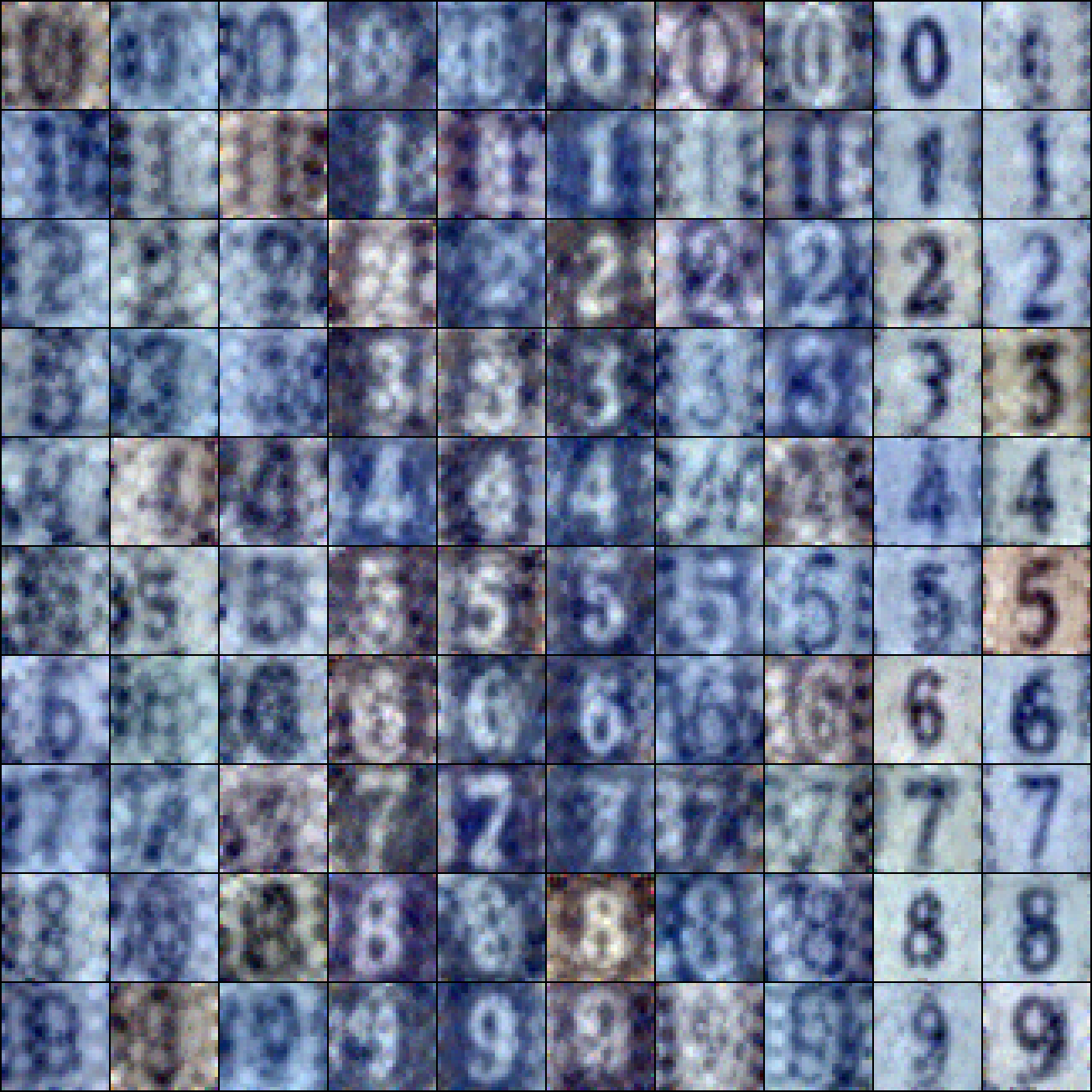}}\hfill
    \subfloat[Y]{\includegraphics[width=0.24\linewidth]{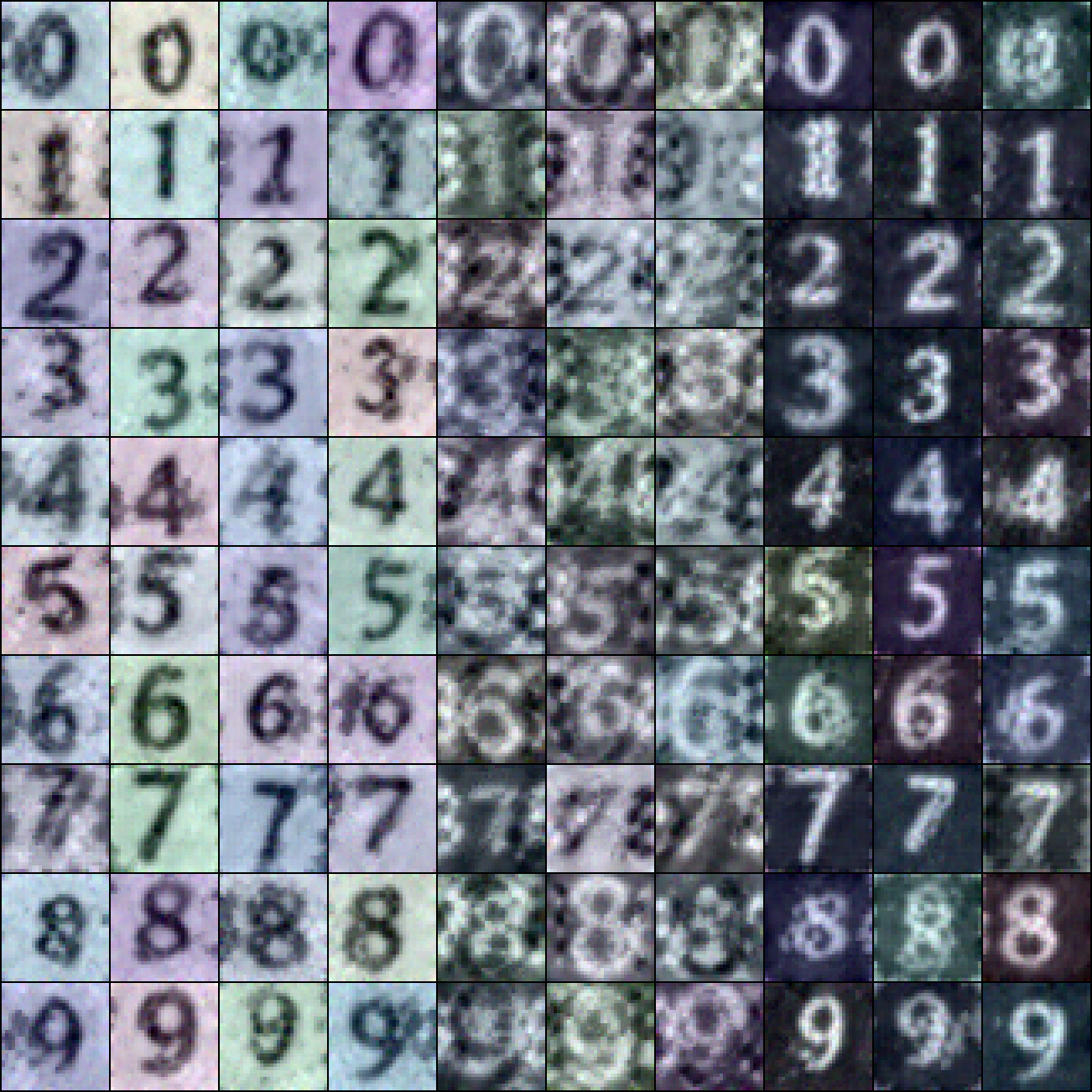}}
    \caption{Visualization of distilled images on the Digits-DG under the SDG setting (IPC=10).}    
    \label{fig:sdg_digit}
\end{figure} 

\noindent\textbf{Quantitative Results on PACS.}\quad 
Table~\ref{tab:supp_sdg} presents the SDG results on the PACS dataset (IPC=10, 20). Consistent with our main findings, SGS consistently outperforms baseline methods, demonstrating its strong generalization capability even when trained on a single source domain.

\begin{table}[h!]
    \vspace{-2mm}
    \caption{Single-domain generalization performance on PACS with IPC=10 and IPC=20.}
    \centering
    \resizebox{0.97\linewidth}{!}{
    \begin{tabular}{c|ccccc|ccccc}
    \toprule
    IPC & \multicolumn{5}{c|}{10} & \multicolumn{5}{c}{20} \\
    Source & P & A & C & S & Avg. & P & A & C & S & Avg. \\ \midrule
    \multicolumn{1}{l|}{DM} 
    & 23.6\pmval0.8 & 36.7\pmval0.6 & 32.5\pmval0.8 & 21.7\pmval1.5 & 28.6\pmval0.9
    & 27.0\pmval0.9 & 37.9\pmval1.4 & 34.5\pmval0.6 & 19.3\pmval0.8 & 29.7\pmval0.9\\

    \rowcolor[gray]{0.9}
    \multicolumn{1}{l|}{\quad+SGS} 
    & \textcolor{red}{\textbf{29.0}}\pmval0.7 & \textcolor{red}{\textbf{39.6}}\pmval1.4 & \textcolor{red}{33.7}\pmval1.1 & \textcolor{red}{\textbf{22.5}}\pmval1.0 & \textcolor{red}{\textbf{31.2}}\pmval1.1
    & \textcolor{red}{\textbf{28.2}}\pmval1.0 & \textcolor{red}{\textbf{38.3}}\pmval1.0 & \textcolor{blue}{34.2}\pmval0.9 & \textcolor{red}{20.4}\pmval0.8 & \textcolor{red}{30.3}\pmval0.9\\
    
    \multicolumn{1}{l|}{HDD-DM} 
    & 25.1\pmval0.4 & 35.7\pmval0.9 & 31.5\pmval0.9 & 21.2\pmval0.2 & 28.4\pmval0.6
    & 24.3\pmval0.9 & 36.4\pmval1.4 & 34.7\pmval0.8 & 19.5\pmval1.4 & 28.7\pmval1.1 \\
    
    \rowcolor[gray]{0.9}
    \multicolumn{1}{l|}{\quad+SGS} 
    & \textcolor{red}{27.2}\pmval0.8 & \textcolor{red}{37.4}\pmval1.5 & \textcolor{red}{\textbf{34.6}}\pmval0.4 & \textcolor{red}{\textbf{22.5}}\pmval1.3 & \textcolor{red}{30.4}\pmval1.0
    & \textcolor{red}{26.9}\pmval1.3 & \textcolor{red}{37.9}\pmval1.6 & \textcolor{red}{\textbf{36.5}}\pmval0.6 & \textcolor{red}{\textbf{23.4}}\pmval1.0 & \textcolor{red}{\textbf{31.2}}\pmval1.1\\
    \bottomrule
    \end{tabular}
    }
    \label{tab:supp_sdg}
\end{table} 

\section{Initialization Strategy}
We study performance variations under different initialization strategies: \textit{Noise}, \textit{Random}, and \textit{Uniform}.
\textit{Noise} initializes synthetic images with random noise.
\textit{Random} assigns synthetic images by randomly sampling from the entire source dataset.
\textit{Uniform} distributes synthetic images evenly across each domain dataset.
Tables~\ref{tab:init} presents the results on Digits-DG and PACS, respectively, for IPC = 10 and IPC = 20.
We observe that \textit{Uniform} consistently outperforms the other strategies on average across all settings.
We attribute this improvement to increased diversity: uniform allocation helps prevent overfitting to the characteristics of any single domain.
\begin{table}[h!]
    \vspace{-4mm}
    \centering
    \caption{Performance changes on Digits-DG according to different initialization strategies.}
    \resizebox{0.84\linewidth}{!}{
    \setlength{\tabcolsep}{3pt}
    \begin{tabular}{c|ccccc|ccccc}
    \toprule
    Dataset & \multicolumn{10}{c}{Digits-DG} \\
    \multicolumn{1}{c|}{IPC} & \multicolumn{5}{c|}{10} & \multicolumn{5}{c}{20} \\ \midrule
    Target & M & N & V & Y & Avg. & M & N & V & Y & Avg. \\
    \midrule
     Noise &
     89.9\pmval0.5 & 50.6\pmval1.6 & 52.5\pmval1.1 & \textbf{71.4}\pmval0.5 & 66.1\pmval0.9 &
    89.8\pmval0.6 & 53.2\pmval1.6 & \textbf{58.2}\pmval1.0 & 73.6\pmval0.8 & 68.7\pmval1.0 \\
    
    Random & 
     89.4\pmval0.8 & 51.4\pmval1.9 & 53.9\pmval1.1 & 70.1\pmval1.2 & 66.2\pmval1.3 &
    89.5\pmval0.5 & 52.8\pmval0.7 & 57.4\pmval1.3 & \textbf{75.1}\pmval0.2 & 68.7\pmval0.7 \\
   
    Uniform & 
    \textbf{90.3}\pmval0.4 & \textbf{54.1}\pmval0.9 & \textbf{55.3}\pmval0.4 & 71.0\pmval0.8 & \textbf{67.7}\pmval0.6 &
    \textbf{90.7}\pmval0.5 & \textbf{54.3}\pmval1.2 & 58.0\pmval0.5 & 74.3\pmval0.8 & \textbf{69.3}\pmval0.8 \\
  
    \bottomrule
    \multicolumn{11}{c}{} \\
    \toprule
    
    Dataset & \multicolumn{10}{c}{PACS} \\ 
    \multicolumn{1}{c|}{IPC} & \multicolumn{5}{c|}{10} & \multicolumn{5}{c}{20} \\ \midrule
    Target & P & A & C & S & Avg. & P & A & C & S & Avg. \\
    \midrule
     Noise &
     52.8\pmval0.7 & 27.6\pmval0.6 & 33.4\pmval1.7 & \textbf{37.8}\pmval4.6 & 37.9\pmval1.9 &
    52.6\pmval1.3 & \textbf{31.1}\pmval0.7 & 35.1\pmval1.3 & 37.7\pmval2.5 & 39.1\pmval1.5 \\
    
    Random & 
     52.6\pmval3.0 & \textbf{28.4}\pmval1.1 & 31.6\pmval1.2 & 34.1\pmval3.7 & 36.7\pmval2.3 &
    52.9\pmval0.7 & 28.6\pmval0.8 & 34.1\pmval0.8 & \textbf{40.1}\pmval1.4 & 38.9\pmval0.9 \\
    
    Uniform & 
     \textbf{53.8}\pmval0.4 & \textbf{28.4}\pmval1.1 & \textbf{35.1}\pmval1.5 & 37.2\pmval1.8 & \textbf{38.6}\pmval1.2 &
    \textbf{53.2}\pmval1.0 & 29.1\pmval0.6 & \textbf{37.7}\pmval2.3 & 39.8\pmval2.5 & \textbf{40.0}\pmval1.6 \\
  
    \bottomrule
    \end{tabular}}
    \label{tab:init}
\end{table} 

\section{More Results of Ablation Study}
We provide an extended ablation study on the Digits-DG and PACS to further validate the contribution of each gradient component (Table~\ref{tab:ablation_sup}). 

The results consistently show that our full formulation ($g^\text{class} + g^\text{domain} + g$) achieves the highest OOD accuracy across all IPC settings on average.
Notably, the significant performance collapse when using only the domain-specific component ($g^\text{domain}$) confirms that it successfully isolates non-essential style variations from the class-discriminative signals. The synergistic effect of combining all components ensures that SGS effectively filters out domain-specific signals while preserving domain-invariant ones, leading to robust generalization regardless of the memory budget.
\begin{table}[h!]
    \centering
    \caption{Ablation study on Digits-DG and PACS when IPC=10 and IPC=20.}
    \resizebox{0.96\linewidth}{!}{
    \setlength{\tabcolsep}{2pt}
    \begin{tabular}{ccc|ccccc|ccccc}
    \toprule
    \multicolumn{3}{c|}{Dataset} & \multicolumn{10}{c}{Digits-DG} \\ 
    \multicolumn{3}{c|}{IPC} & \multicolumn{5}{c|}{10} & \multicolumn{5}{c}{20} \\ \midrule
    $\boldsymbol{g}^\text{class}$ & $\boldsymbol{g}^\text{domain}$ & $\boldsymbol{g}$ & M & N & V & Y & Avg. & M & N & V & Y & Avg. \\
    \midrule
     &  & \checkmark &
    85.3\pmval0.7 & 51.4\pmval1.5 & 53.6\pmval0.9 & 69.9\pmval0.9 & 65.1\pmval1.0 &
    88.4\pmval0.8 & 53.1\pmval1.0 & 56.4\pmval1.1 & 73.7\pmval0.2 & 67.9\pmval0.8\\
    
    \checkmark & & & 
    85.5\pmval0.9 &52.6\pmval1.2 &52.4\pmval1.3 &68.9\pmval1.5 &64.9\pmval1.2 &
    88.6\pmval0.9 & 52.6\pmval1.3 & 56.5\pmval1.2 & 74.2\pmval1.5 & 68.0\pmval1.2 \\
   
    & \checkmark & & 
    65.8\pmval2.5 &42.5\pmval1.0 &27.2\pmval0.7 &43.1\pmval1.2 &44.7\pmval1.4 &
    68.9\pmval5.3 & 43.6\pmval1.0 & 37.7\pmval0.8 & 52.8\pmval1.2 & 50.8\pmval2.1 \\
    
    \checkmark & \checkmark & & 
    88.6\pmval0.9 &51.8\pmval1.0 &52.7\pmval0.8 &67.1\pmval0.5 &65.1\pmval0.8 &
    86.9\pmval0.9 & 51.2\pmval1.7 & 57.0\pmval0.8 & 72.1\pmval0.5 & 66.8\pmval1.0 \\
    
    \rowcolor[gray]{0.9}        
    \checkmark & \checkmark & \checkmark & 
    \textbf{90.3}\pmval0.4 &\textbf{54.1}\pmval0.9 &\textbf{55.3}\pmval0.4 &\textbf{71.0}\pmval0.8 &\textbf{67.7}\pmval0.6 &
    \textbf{90.7}\pmval0.5 & \textbf{54.3}\pmval1.2 & \textbf{58.0}\pmval0.5 & \textbf{74.3}\pmval0.8 & \textbf{69.3}\pmval0.8 \\
    \bottomrule
    
    \multicolumn{13}{c}{} \\ 
    
    \toprule
    \multicolumn{3}{c|}{Dataset} & \multicolumn{10}{c}{PACS} \\ 
    \multicolumn{3}{c|}{IPC} & \multicolumn{5}{c|}{10} & \multicolumn{5}{c}{20} \\ \midrule
    $\boldsymbol{g}^\text{class}$ & $\boldsymbol{g}^\text{domain}$ & $\boldsymbol{g}$ & P & A & C & S & Avg. & P & A & C & S & Avg. \\
    \midrule
     &  & \checkmark &
    48.0\pmval0.7 & 27.9\pmval0.5 & 33.8\pmval1.5 & 32.6\pmval2.2 & 35.6\pmval1.2 &
    49.2\pmval0.4 & 28.8\pmval1.0 & 35.4\pmval2.2 & 38.1\pmval2.7 & 37.9\pmval1.6 \\   
    
    \checkmark & & & 
    49.4\pmval2.6 & 27.7\pmval0.6 & \textbf{35.9}\pmval0.8 & 33.0\pmval2.2 & 36.5\pmval1.6 &
    52.6\pmval0.7 & \textbf{29.1}\pmval0.1 & 37.6\pmval1.4 & 37.8\pmval1.8 & 39.3\pmval1.0 \\
   
    & \checkmark & & 
    30.6\pmval3.1 & 24.6\pmval0.7 & 29.5\pmval0.8 & 24.5\pmval3.9 & 27.3\pmval2.1 &
    39.2\pmval2.7 & 23.4\pmval0.9 & 30.2\pmval1.1 & 30.7\pmval3.4 & 30.9\pmval2.0 \\
    
    \checkmark & \checkmark & & 
    \textbf{54.7}\pmval1.0 & 25.6\pmval0.4 & 35.1\pmval0.9 & 33.9\pmval3.5 & 37.3\pmval1.5 &
    51.7\pmval1.2 & 29.0\pmval0.6 & 31.9\pmval1.5 & 38.5\pmval0.7 & 37.8\pmval1.0 \\
    
    \rowcolor[gray]{0.9}        
    \checkmark & \checkmark & \checkmark & 
    53.8\pmval0.4 & \textbf{28.4}\pmval1.1 & 35.1\pmval1.5 & \textbf{37.2}\pmval1.8 & \textbf{38.6}\pmval1.2 & 
    \textbf{53.2}\pmval1.0 & \textbf{29.1}\pmval0.6 & \textbf{37.7}\pmval2.3 & \textbf{39.8}\pmval2.5 & \textbf{40.0}\pmval1.6 \\
    \bottomrule
    \end{tabular}}
    \label{tab:ablation_sup}
\end{table}

\section{More Results of Post-hoc Domain Generalization}
Figure~\ref{fig:more_dg} compares the average OOD accuracy and downstream training time of various post-hoc DG methods applied to DM-distilled datasets when IPC=10. 
Existing post-hoc methods reveal a strict trade-off between accuracy and efficiency. \textit{MixStyle} is computationally light but consistently underperforms across all benchmarks. Conversely, augmentation-based methods like \textit{FACT} and \textit{AdvFreq} occasionally achieve competitive accuracy, but they suffer from high performance variance across datasets. More critically, their complex augmentation pipelines introduce severe computational overhead, slowing down the downstream training by nearly 2$\times$ (\textit{FACT}) and 7$\times$ (\textit{AdvFreq}). This directly undermines the core advantage of dataset distillation: accelerating downstream training.
In contrast, our SGS inherently encodes domain-invariant semantics directly into the synthetic pixels during the distillation phase, eliminating the need for expensive post-hoc augmentations. As a result, SGS achieves the highest and most stable generalization performance while perfectly maintaining the extremely efficient training speed of the pure DM baseline.

\begin{figure}[h!]
    \centering
    \subfloat[Digits-DG]{\includegraphics[width=0.33\linewidth]{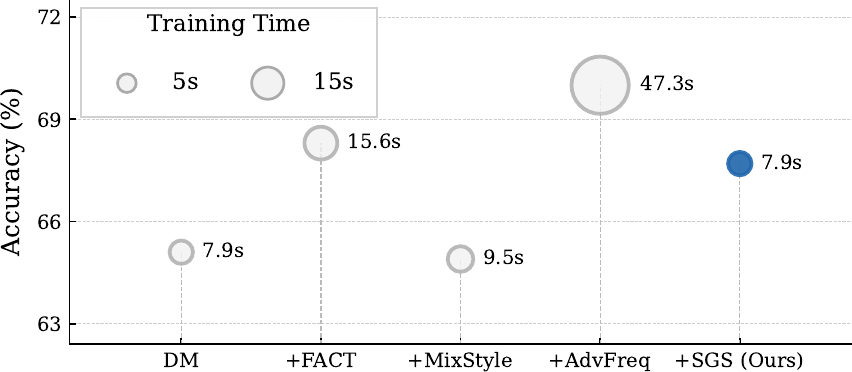}}~
    \subfloat[PACS]{\includegraphics[width=0.33\linewidth]{fig/bubble_pacs.pdf}}~
    \subfloat[CORe50-Hard]{\includegraphics[width=0.33\linewidth]{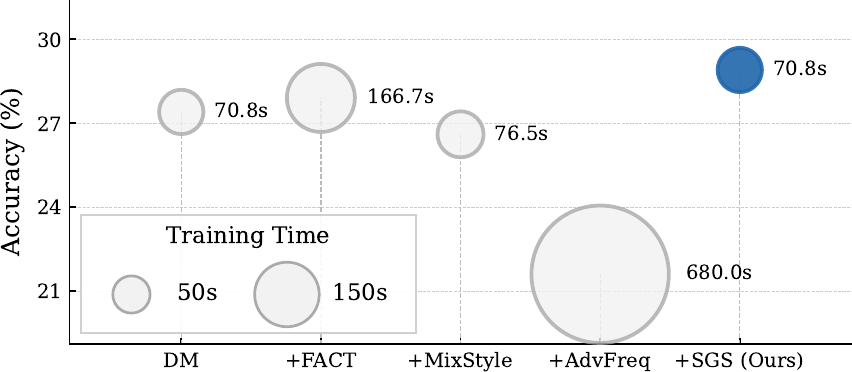}}
    \caption{Comparison of average OOD accuracy and downstream training time among various post-hoc DG methods applied to DM-distilled datasets~(IPC=10).}    \label{fig:more_dg}
\end{figure}

\newpage
\section{Visualization of Distilled Images}\vspace{-2mm}
Figure~\ref{fig:vis_distilled_mdg} and Figure~\ref{fig:vis_distilled_sdg} visualizes distilled images via SGS upon DM on three benchmark datasets when IPC=10 under MDG and SDG settings, respectively.
\begin{figure}[h!]
    \vspace{-2mm}
    \centering
    \subfloat[Target: M]{\includegraphics[width=0.25\linewidth]{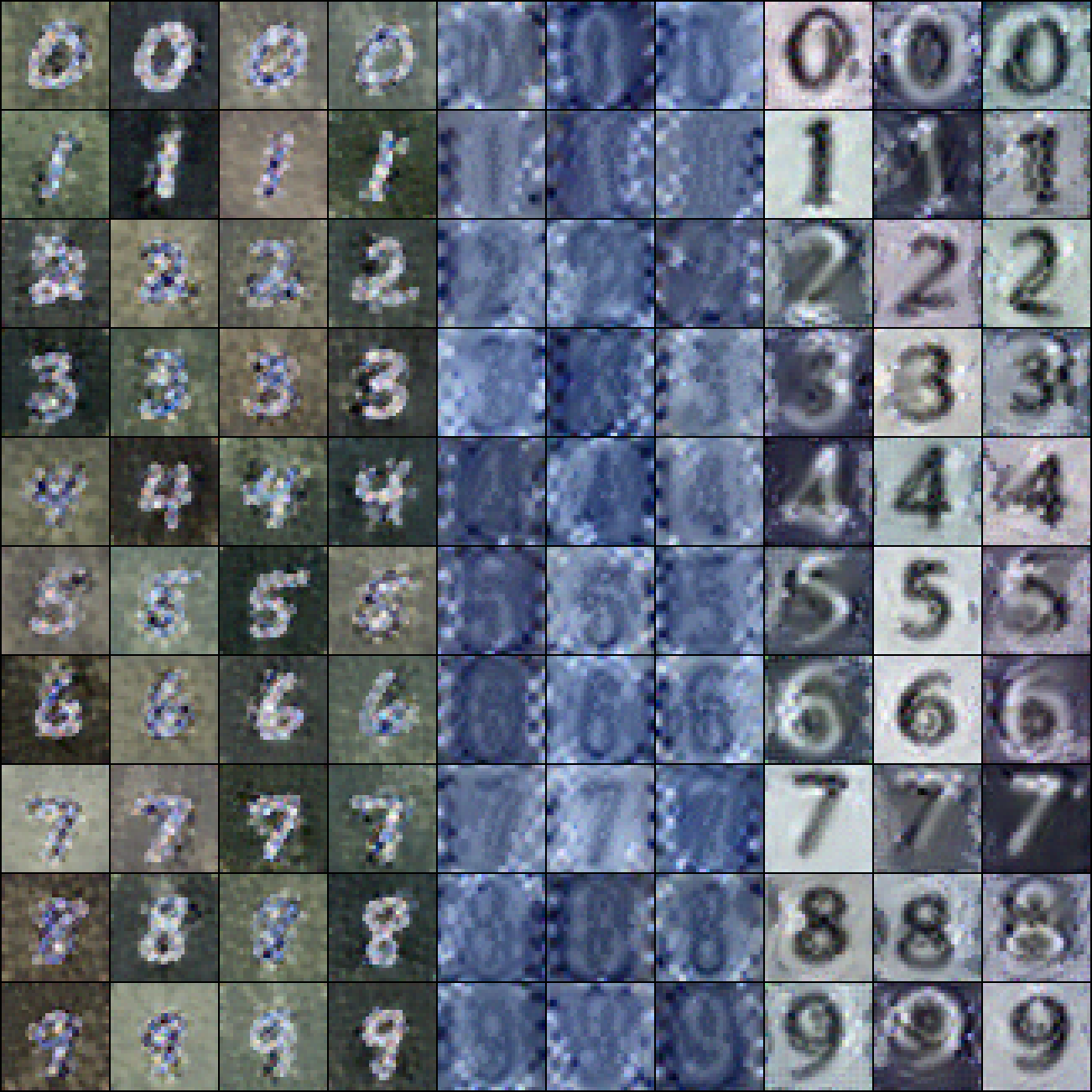}}~
    \subfloat[Target: N]{\includegraphics[width=0.25\linewidth]{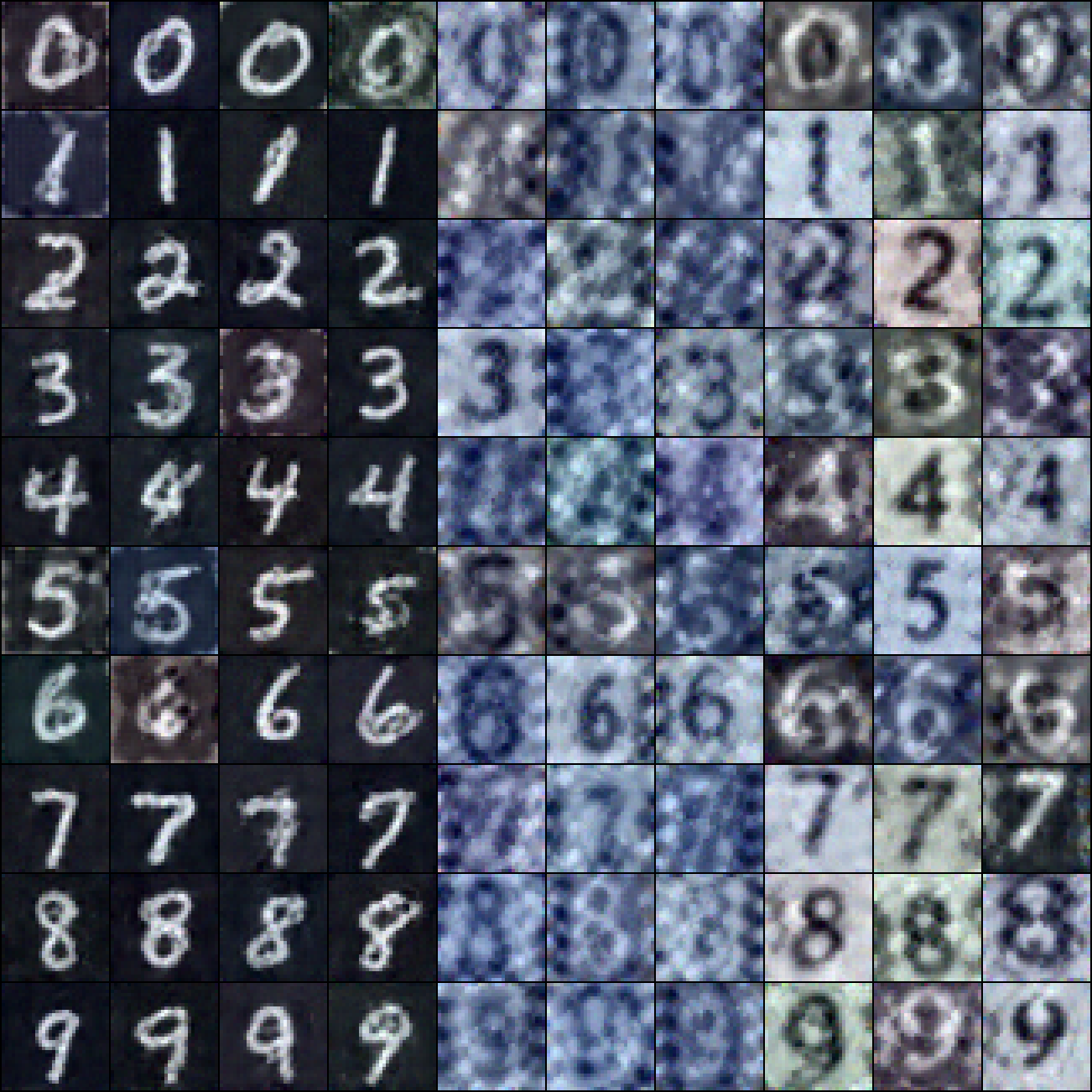}}~
    \subfloat[Target: V]{\includegraphics[width=0.25\linewidth]{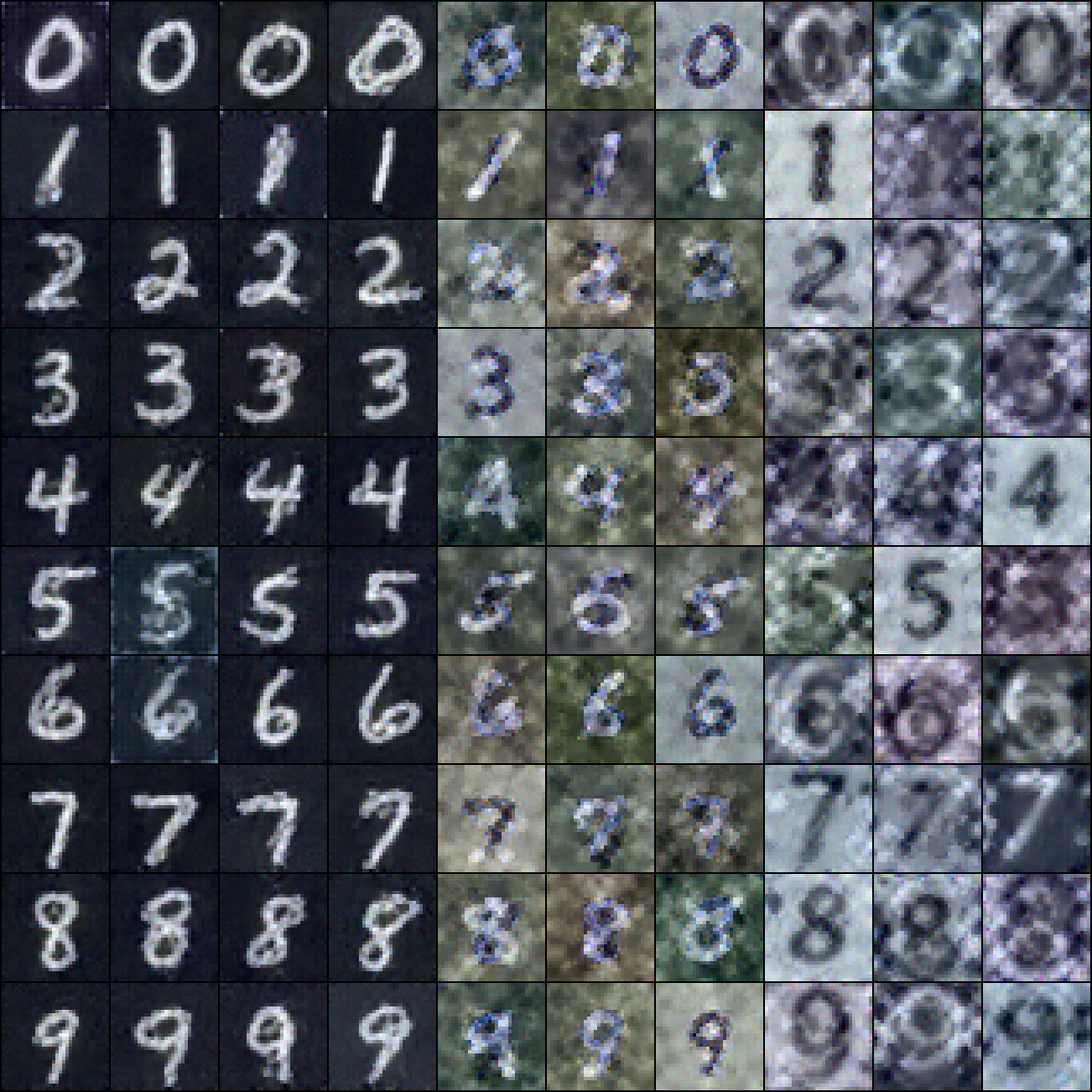}}~
    \subfloat[Target: Y]{\includegraphics[width=0.25\linewidth]{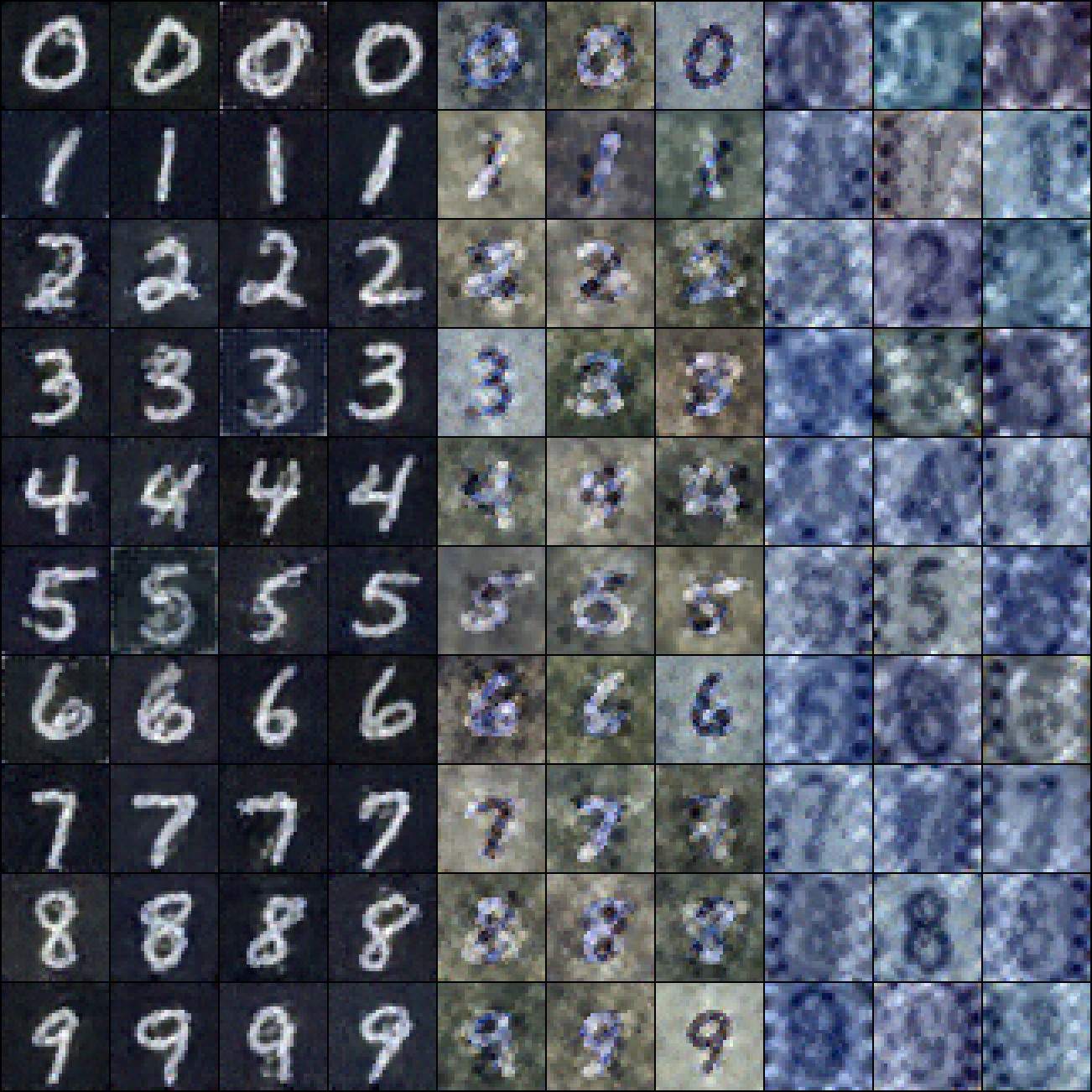}}
    \\
    \subfloat[Target: P]{\includegraphics[width=0.25\linewidth]{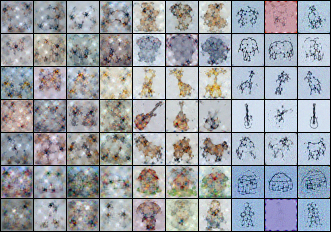}}~
    \subfloat[Target:A]{\includegraphics[width=0.25\linewidth]{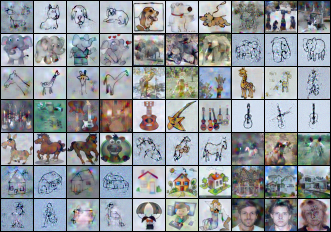}}~
    \subfloat[Target:C]{\includegraphics[width=0.25\linewidth]{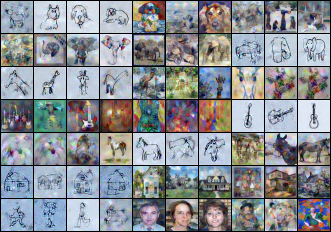}}~
    \subfloat[Target:S]{\includegraphics[width=0.25\linewidth]{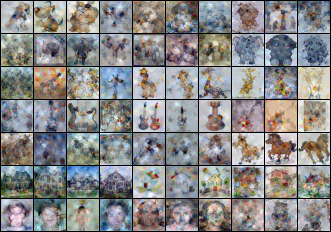}}
    \\
    \subfloat[Target: s5]{\includegraphics[width=0.25\linewidth]{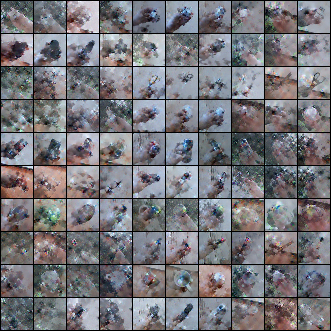}}~
    \subfloat[Target: s7]{\includegraphics[width=0.25\linewidth]{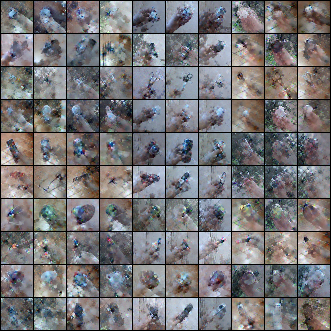}}~
    \subfloat[Target: s9]{\includegraphics[width=0.25\linewidth]{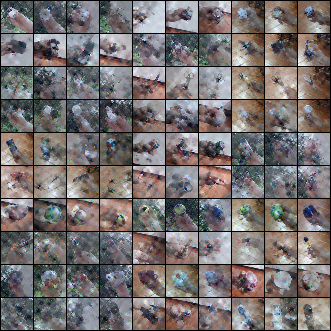}}~
    \subfloat[Target: s10]{\includegraphics[width=0.25\linewidth]{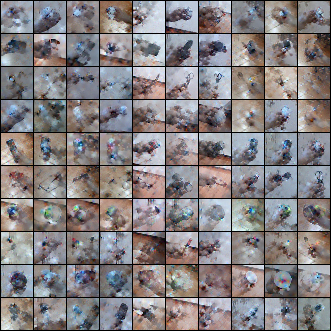}}

    \caption{Distilled Images of our proposed SGS upon DM when IPC=10 under MDG setting. Each row indicates Digits-DG, PACS, and CORe50-Hard, respectively, on each target domain.}
    \label{fig:vis_distilled_mdg}
\end{figure}

\begin{figure}[h!]
    \centering
    \subfloat[Source: M]{\includegraphics[width=0.25\linewidth]{fig/SGS_syn/SDG/Digits-DG/M.png}}~
    \subfloat[Source: N]{\includegraphics[width=0.25\linewidth]{fig/SGS_syn/SDG/Digits-DG/N.png}}~
    \subfloat[Source: V]{\includegraphics[width=0.25\linewidth]{fig/SGS_syn/SDG/Digits-DG/V.png}}~
    \subfloat[Source: Y]{\includegraphics[width=0.25\linewidth]{fig/SGS_syn/SDG/Digits-DG/Y.png}}
    \\
    \subfloat[Source: P]{\includegraphics[width=0.25\linewidth]{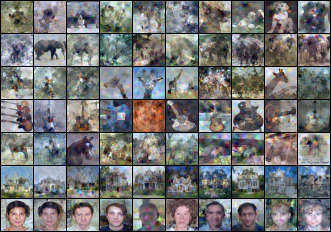}}~
    \subfloat[Source: A]{\includegraphics[width=0.25\linewidth]{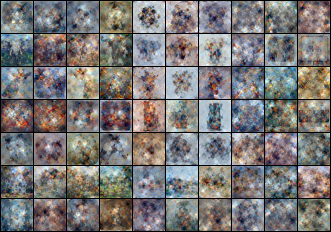}}~
    \subfloat[Source: C]{\includegraphics[width=0.25\linewidth]{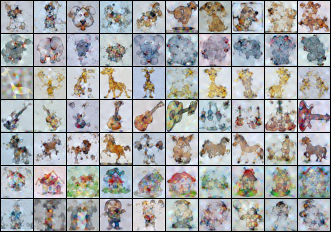}}~
    \subfloat[Source: S]{\includegraphics[width=0.25\linewidth]{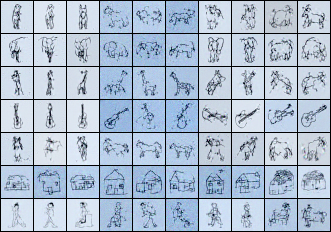}}
    \caption{Distilled Images of our proposed SGS upon DM when IPC=10 under SDG setting. Each row indicates Digits-DG, and PACS, respectively, on each source domain.}
    \label{fig:vis_distilled_sdg}
\end{figure}


\newpage
\end{document}